\PassOptionsToPackage{table,x11names}{xcolor}
\PassOptionsToPackage{breaklinks, colorlinks}{hyperref}
\documentclass[11pt]{article}
\usepackage[round]{natbib}
\newif\ifarxiv
\arxivtrue

\usepackage[utf8]{inputenc} % allow utf-8 input
\usepackage[T1]{fontenc}    % use 8-bit T1 fonts
\usepackage{url}            % simple URL typesetting
\usepackage{booktabs, pgfplots}
\usepackage{xcolor}
\usepackage{multirow}
\usepackage[normalem]{ulem}
\usepackage{siunitx}
\usepackage{amsfonts}       % blackboard math symbols
\usepackage{nicefrac}       % compact symbols for 1/2, etc.
\usepackage{microtype}      % microtypography
\usepackage{graphicx}
\usepackage{subcaption}
\usepackage{makecell}
\usepackage{wrapfig}
\usepackage{enumitem}
\usepackage{caption}
\usepackage{algorithm}
\usepackage{algorithmic}
\usepackage{float}

\usepackage{amsmath,amssymb,amsthm}
\usepackage{mathtools}

\usepackage{hyperref}
\let\cite=\citep

\definecolor{citegreen}{rgb}{0.0, 0.5, 0.0}
\definecolor{neuripsblue}{RGB}{25, 25, 112}
\hypersetup{
    colorlinks = true,
    linkcolor=red,
    citecolor=blue,
    urlcolor=black,
}
\pgfplotsset{compat=1.18}

\newcommand{\D}{\mathcal{D}}
\newcommand{\Z}{\mathcal{Z}}
\newcommand{\Zval}{\mathcal{Z}_{\val}}
\newcommand{\Zpool}{\mathcal{Z}_{\mathrm{pool}}}
\newcommand{\val}{\mathrm{val}}

\theoremstyle{definition}
\newtheorem{theorem}{Theorem}[section]
\newtheorem{definition}[theorem]{Definition}

\newtheorem{proposition}[theorem]{Proposition}

\newtheorem{lemma}[theorem]{Lemma}

\theoremstyle{remark}

\usepackage{tacs_preprint}

\title{Let the Target Select for Itself: Data Selection via Target-Aligned Paths}
\author{
  \textbf{Huitao Yang} \qquad \textbf{Hengzhi He} \qquad \textbf{Tung Sum Thomas Kwok} \qquad \textbf{Guang Cheng} \\
  {\normalsize University of California, Los Angeles} \\
  \texttt{\{htyang03,hengzhihe,tk1018,guangcheng\}@ucla.edu}
}
\hypersetup{
  pdftitle={Let the Target Select for Itself: Data Selection via Target-Aligned Paths},
  pdfauthor={Huitao Yang, Hengzhi He, Tung Sum Thomas Kwok, Guang Cheng},
}

\begin{document}
\maketitle

\begin{abstract}
Targeted data selection seeks training examples from a candidate pool that improve downstream task performance. While trajectory-based selectors effectively guide this process, existing approaches typically construct reference states by warming up on the candidate pool itself, thereby inheriting pool-dependent distributional biases. To address this coupling, we introduce Target-Aligned Candidate Selection (TACS), which constructs a short, capacity-constrained reference path using a compact target-validation proxy. Candidates are ranked by their normalized loss reduction along this trajectory, requiring only forward passes and allowing the path to be reused across distinct pools for a fixed model-target pair. Empirically, TACS achieves competitive downstream performance across controlled logistic, vision, and NLP benchmarks, while slashing selection compute by up to 69\% and candidate storage to less than 10 MB in large-scale instruction tuning.
\end{abstract}
\begin{quote}
\noindent\textbf{Project Page:} \url{https://github.com/davidyht/TACS}
\end{quote}

\section{Introduction}\label{sec:intro}
Targeted data selection identifies training examples from an uncurated candidate pool to maximize downstream performance on a specific task~\cite{xia2024lessselectinginfluentialdata,killamsetty2021glistergeneralizationbaseddata,killamsetty2021gradmatchgradientmatchingbased}. Many scalable selectors follow a two-stage paradigm: warming up a base model to construct reference checkpoints, then scoring candidate--target interactions along this trajectory~\cite{pruthi2020estimatingtrainingdatainfluence,xia2024lessselectinginfluentialdata,jain2025trainvalidationtovfast,min2026gisttargeteddataselection}. These reference checkpoints fundamentally define the parameter states where candidate utility is judged.

Because existing warmups train on the candidate pool, the reference states evaluating candidates are themselves shaped by pool data. We find that this coupling introduces substantial variability: altering the warmup trajectory while holding candidate pool, target proxy, and scoring rule fixed markedly shifts rankings and subset composition. We term this phenomenon \emph{reference-path dependence}: varying warmup data alone causes top-ranked subset overlap to drop to 7.0\% and downstream F1 to fluctuate by over 5 points. Hence, reference trajectories act as active selection priors (Section~\ref{sec:motivation:observations}).

To address these limitations, we propose Target-Aligned Candidate Selection (TACS), which grounds reference states entirely in the target task while treating the trajectory as an active design variable. Starting from the base model, TACS optimizes a short, capacity-constrained path on a compact target-validation proxy, calibrating its duration before scoring (Figure~\ref{fig:reference_trajectory_concept}). Candidates are then ranked by forward-only normalized loss reduction along this path. Crucially, decoupling path construction from candidate data enables amortizing a single target path across multiple distinct pools.

Theoretically, a linearized model shows how reference-path uncertainty propagates through an attribution rule into selection variability and downstream risk. Pool warmup can provide useful pool-specific information, but also makes selection inherit variation in the warmup composition; a target-induced reference removes this source of pool-dependent uncertainty. We further characterize how scoring rules impose distinct geometric preferences over candidate distributions and how finite target paths spectrally emphasize faster-learned target directions.

Empirically, controlled experiments reveal behaviors consistent with this analysis: target-path selection produces more target-aligned retraining directions and more strongly filters corrupted candidates, while trajectory-depth ablations show that the useful extent of target adaptation varies substantially across tasks. At scale, across four instruction pools and three downstream tasks, TACS remains competitive with LESS and ToV while reducing selection GPU-hours by 69\% and 53\%, respectively.
Our main contributions are:
\begin{enumerate}[leftmargin=*, itemsep=0pt, topsep=2pt]
\item \textbf{Reference-path dependence}: We identify and quantify the sensitivity of trajectory-based data selection to reference construction, showing that changing the warmup data alone can substantially alter selected subsets and downstream performance (Sec.~\ref{sec:motivation}).
\item \textbf{TACS framework}: We introduce a target-guided, capacity-constrained reference path that decouples candidate evaluation from pool warmup, enabling forward-only scoring and path reuse across candidate pools (Sec.~\ref{sec:methodology}).
\item \textbf{Theoretical \& empirical analysis}: We characterize how reference uncertainty propagates to selection, analyze the geometric preferences and learning dynamics of target-induced paths, and evaluate TACS across controlled tasks and instruction tuning (Sec.~\ref{sec:theory},~\ref{sec:empirical_study}).
\end{enumerate}

\begin{figure}[t]
\centering
\begin{minipage}[c]{0.423\textwidth}
\centering
\includegraphics[width=\linewidth]{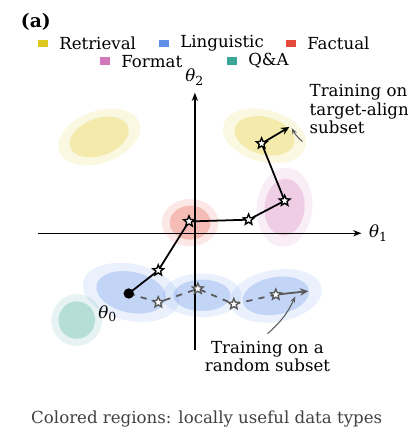}
\end{minipage}\hspace{0.04\textwidth}
\begin{minipage}[c]{0.423\textwidth}
\centering
\includegraphics[width=\linewidth]{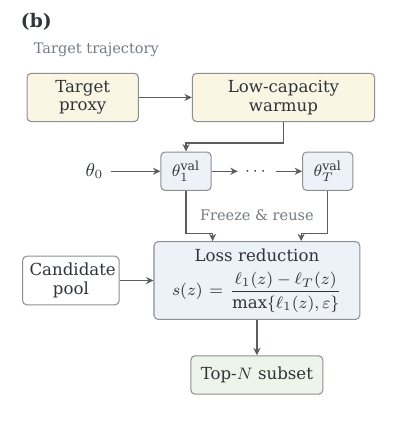}
\end{minipage}
\caption{\textbf{Reference trajectories determine where candidates are evaluated.}
\textbf{(a)} Parameter-space trajectories define the intermediate checkpoints where candidate utility is scored: a generic warmup path can miss task-relevant parameter regions, whereas a target-aligned path guides the model into states where target-critical data signals become visible.
\textbf{(b)} TACS constructs this reference trajectory directly on a target-validation proxy, enabling forward-only candidate scoring and cross-pool path amortization.}
\label{fig:reference_trajectory_concept}
\end{figure}

\section{Why the Reference Path Matters}\label{sec:motivation}
\subsection{Problem Setting}
\label{sec:motivation:setting}

Targeted data selection seeks a subset $S \subseteq \Zpool$ of budget $|S| = N$ from candidate pool $\Zpool$ to minimize target risk $R(\theta(S)) = \mathbb{E}_{z \sim \D}[\ell(\theta(S); z)]$, where $\theta(S) \in \arg\min_{\theta \in \Theta} R_S(\theta)$ minimizes empirical risk $R_S(\theta) = \frac{1}{|S|}\sum_{z \in S} \ell(\theta; z)$. Because the target distribution $\D$ is inaccessible during selection, an empirical validation set $\Zval \sim \mathcal{D}$ serves as a proxy via $R_{\val}(\theta) = \frac{1}{|\Zval|}\sum_{z \in \Zval}\ell(\theta; z)$, while final generalization is evaluated on disjoint, held-out target test sets.

\begin{figure}[t]
    \centering
    \includegraphics[width=0.6\textwidth]{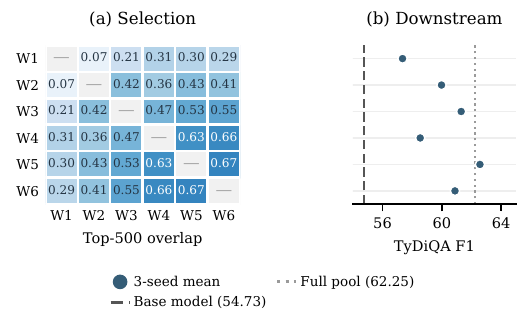}
    \caption{\textbf{Reference-path dependence in LESS.} 
    Holding candidate pool, target proxy, and scoring rule fixed, varying only the warmup data subset yields divergent selections. 
    \textbf{(a)} Pairwise top-500 candidate overlap. 
    \textbf{(b)} Downstream TyDiQA F1 across warmup paths (lines denote base model and full-pool baselines).}
    \label{fig:motivation:path_sensitivity}
\end{figure}

\subsection{Trajectory-Conditioned Selection}
\label{sec:motivation:trajectory_attribution}

For a subset $S$, consider retraining under gradient flow, $\dot\theta_t^S=-\eta_t\nabla R_S(\theta_t^S)$, initialized from a shared base model $\theta_0$. Applying the chain rule to the validation risk yields
\begin{equation}
\label{eq:motivation:subset_decomposition}
    R_{\val}(\theta_0)-R_{\val}(\theta_T^S)
    =\frac{1}{|S|}\sum_{z\in S}\int_0^T\eta_t
    \left\langle
        \nabla R_{\val}(\theta_t^S),
        \nabla\ell(\theta_t^S;z)
    \right\rangle\,\mathrm dt
    =:\frac{1}{|S|}\sum_{z\in S}F(z,\mathcal T_S).
\end{equation}
Minimizing final validation risk thus translates to maximizing cumulative candidate--target gradient alignment along the retraining trajectory $\mathcal T_S=\{\theta_t^S\}_{t=0}^T$. In this ideal decomposition, the selected subset $S$ simultaneously dictates the examples evaluated and the trajectory along which evaluation occurs.

Because $\mathcal T_S$ is unknown prior to selection, practical selectors decouple this dependency by substituting $\mathcal T_S$ with an external reference path $\mathcal T_{\mathrm{ref}}$. A Riemann sum over discrete checkpoints along $\mathcal T_{\mathrm{ref}}$ approximates the integral $F(z,\mathcal T_{\mathrm{ref}})$. For instance, LESS accumulates optimizer-adjusted gradient cosine similarities~\cite{xia2024lessselectinginfluentialdata}, while ToV measures localized loss changes induced by validation updates~\cite{jain2025trainvalidationtovfast}. 
Consequently, candidate rankings become explicitly conditioned on $\mathcal T_{\mathrm{ref}}$. We refer to this sensitivity as \emph{reference-path dependence}: altering the reference path shifts the evaluation states $\{\theta_t\}$, thereby altering candidate scores even with the pool, target proxy, and scoring rule held constant.

\subsection{Observations}
\label{sec:motivation:observations}

\paragraph{Selection is sensitive to the reference path.}
To quantify this sensitivity, we vary the warmup data across six trajectories in LESS while holding the pool, target proxy, and scoring rule constant.
As shown in Figure~\ref{fig:motivation:path_sensitivity}, warmup variation causes severe divergence: pairwise top-500 selection overlap spans 0.070 to 0.674 (median 0.424), while full-pool rank correlations range from 0.100 to 0.843 (median 0.552). This divergence directly impacts downstream utility, with retrained TyDiQA F1 fluctuating between 57.35 and 62.57 (against 62.25 for full-pool training). Thus, reference paths act as active selection priors rather than inert scaffolds, establishing path construction as a critical design choice.

\paragraph{The target distribution as an idealized reference.}
To formalize the reference choice, consider an idealized setting that selects a training distribution $Q$ from an admissible family $\mathcal{Q}$ (distributions capable of approximating $\D$). For any minimizer $\theta_Q^\star \in \arg\min_\theta \mathbb{E}_{z\sim Q}[\ell(\theta;z)]$, when $\D \in \mathcal{Q}$, direct target training attains the global lower bound:
\[
    R_{\D}(\theta_{\D}^\star) = \min_{\theta} R_{\D}(\theta) \leq R_{\D}(\theta_Q^\star), \qquad \forall\, Q \in \mathcal{Q}.
\]
Under the motivating premise that a sufficiently expressive pool contains an effective target surrogate, adaptation on $\D$ provides a canonical, pool-agnostic reference. This population inequality does not identify an optimal finite-time reference or a finite subset; it motivates leveraging the available target-validation proxy $\Zval$ to construct target-induced reference trajectories for attribution.

\section{Methodology}\label{sec:methodology}
TACS separates two choices that pool-induced selectors couple: which model states define candidate utility, and which candidates are evaluated at those states. It first constructs a short reference path using only the target proxy, completely decoupled from the candidate data. It then reuses two checkpoints along that path to score any candidate pool with forward-only evaluations. The complete procedure is outlined in Algorithm~\ref{alg:target_selection}\ifarxiv\else{} (Appendix~\ref{app:algorithm})\fi{}.

\ifarxiv
\begin{algorithm}[ht!]
    \caption{\textbf{T}arget-\textbf{A}ligned \textbf{C}andidate \textbf{S}election (TACS).}
    \label{alg:target_selection}
    \small
    \begin{algorithmic}[1]
        \STATE \textbf{Input:} Val set $\Zval$, Pool $\Zpool$, Budget $N$
        \STATE \textbf{Initialize:} Base parameters $\theta_0$
        \STATE Select learning rate $\eta$ and epochs $T$ via cross-validation (Eq.~\ref{eq:specificity})
        \STATE Generate target-aligned warmup trajectory on $\mathcal{Z}_{\val}$, $\mathcal{T}_{\val} = \{\theta^{\val}_0, ..., \theta^{\val}_T\}$, via Eq.~\ref{eq: gd_val}
        \FOR{each candidate $z \in \Zpool$}
            \STATE Compute normalized candidate score $s(z)$ via Eq.~\ref{eqn:norm_final_drop}
        \ENDFOR
        \STATE \textbf{return} Subset $S \subset \Zpool$ maximizing total score
    \end{algorithmic}
\end{algorithm}

\fi

\subsection{Let the Target Define the Reference Path}
\label{sec:trajectory_construction}

Starting from base parameters $\theta_0$, TACS takes a short sequence of updates on the target-validation risk:
\begin{equation}\label{eq: gd_val}
\theta_t^{\val}=\theta_{t-1}^{\val}-\eta\nabla_\theta R_\val(\theta_{t-1}^{\val}),
\qquad
R_\val(\theta)=\frac{1}{|\Zval|}\sum_{z\in\Zval}\ell(\theta;z).
\end{equation}
These updates trace the reference path $\mathcal T_\val=\{\theta_0^\val,\ldots,\theta_T^\val\}$. Because candidate data is not involved in path construction, a single trajectory can be amortized to score multiple candidate pools for a fixed model--target pair.

The path must traverse far enough to expose target-relevant candidate responses without overfitting the compact proxy. For language models, we regularize the warmup by restricting adaptation to a low-rank LoRA subspace~\cite{hu2021loralowrankadaptationlarge}. We jointly calibrate the learning rate $\eta$ and horizon $T$ by testing whether updates learned from one portion of the proxy generalize to held-out target instances beyond generic instruction signals. Specifically, across validation splits $j$, let $g_{W_j}$ denote the endpoint gradient on the warmup split, $g_{H_j}$ on held-out target examples, and $g_P$ on a fixed, task-agnostic reference. We maximize the mean target specificity:
\begin{equation}\label{eq:specificity}
C_j(\eta,T)=\cos(g_{W_j},g_{H_j})-\cos(g_{W_j},g_P).
\end{equation}
The first term rewards target transfer, while the second discounts alignment with generic instruction following. Because $g_P$ is fixed across all tasks and candidate sources, calibration does not reintroduce pool dependence.

\subsection{Measure Candidate Response Along the Path}
\label{sec:compatibility_scoring}

A useful candidate should become easier for the model as it adapts to the downstream task. We capture this response between the first saved checkpoint and the calibrated endpoint:
\begin{equation}
\label{eqn:norm_final_drop}
s(z)=
\frac{\ell(\theta_1^\val;z)-\ell(\theta_T^\val;z)}
{\max\{\ell(\theta_1^\val;z),\varepsilon\}}.
\end{equation}
We evaluate from $\theta_1^\val$ rather than $\theta_0^\val$ because the initial proxy pass may be dominated by generic, surface-level adaptation such as formatting alignment~\cite{frankle2020earlyphaseneuralnetwork}. For token-level objectives, $\ell$ averages negative log-likelihood across prompt-masked response tokens. Normalizing by the initial loss expresses candidate improvement relative to its initial loss scale. Appendix~\ref{app:seq_len} examines the resulting selected-example lengths.
Given budget $N$, TACS selects the top-$N$ candidates:
\(
S^*=\arg\max_{S\subseteq\Zpool,\,|S|=N}\sum_{z\in S}s(z).
\)
Candidate scoring relies strictly on forward passes through the two frozen checkpoints, requiring zero candidate-gradient caching or backpropagation.

\subsection{Why Loss Reduction Measures Path Alignment}
\label{sec:score_interpretation}

The endpoint score summarizes the cumulative alignment along the trajectory rather than an isolated checkpoint update. Under continuous target-gradient flow, the raw loss reduction equates to:
\begin{equation*}
\ell(\theta_1^\val;z)-\ell(\theta_T^\val;z)
=\int_1^T\eta_t
\left\langle\nabla\ell(\theta_t^\val;z),\nabla R_\val(\theta_t^\val)\right\rangle\,\mathrm dt.
\end{equation*}
A candidate's raw loss decrease thus accumulates alignment with the target adaptation field across the path; normalization additionally weights that decrease by initial loss. With discrete optimizer steps, the endpoint difference measures response to the actual parameter displacements. Section~\ref{sec:theory} characterizes which candidate distributions an idealized normalized score prefers and how the path horizon filters target modes.

\section{Theoretical Analysis}\label{sec:theory}
To understand the mechanisms behind reference-path dependence, we study two complementary questions. First, how does variation in the warmup path propagate to the data selected by a trajectory-based attribution rule, and ultimately to downstream risk (\S\ref{sec:theory:reference})? Second, once the reference path is fixed to the target, what candidate distributions do different scoring rules favor, and what target information is captured along a finite path (\S\ref{sec:theory:linear})? We analyze both questions in a population model that replaces finite subsets with an admissible family of training distributions.

\paragraph{Setting.}
Motivated by linearized analyses of regularized LLM fine-tuning~\citep{afzal2026linearizationexplainsfinetuninglarge}, we consider a linearized regression model
$f_w(x) = f_0(x) + \phi(x)^\top w$
with positive-definite covariance
$H = \mathbb{E}[\phi(x)\phi(x)^\top] \succ 0$.
A candidate distribution $P_u$ corresponds to parameter displacement $u$,
with risk
$R_u(w) = \frac{1}{2}\|w-u\|_H^2$,
where $\|z\|_H^2 = z^\top H z$.
The target task $P_q$ starts from initialization $w=0$, and
$D(u)=R_q(u)$ denotes downstream target risk.
For a mixture of candidate distributions, let
$\mu=\mathbb E[u]$ denote its mean optimum.
Starting from $w=0$, continuous gradient flow on this mixture reaches
\[
    w_t(\mu)=B_t\mu,
    \qquad
    B_t=I-e^{-Ht}.
\]

\subsection{Reference-Induced Selection Uncertainty}
\label{sec:theory:reference}

Section~\ref{sec:motivation} showed that changing the warmup data alone can substantially
alter selection and downstream performance. Following the
distribution-selection formulation of Section~\ref{sec:motivation:observations}, we model selection
over an admissible family
$\mathcal Q=\{Q_u:u\in\mathcal U\}$ and ask how variation in the warmup
reference propagates to the selected distribution and its target risk.
For a warmup distribution with mean optimum $\mu$, let
\begin{equation}
    s(u;\mu)
    =
    \mathcal A(u;B_{t_1}\mu,\ldots,B_{t_L}\mu,q),
    \qquad
    v(\mu)
    =
    \operatorname*{arg\,max}_{u\in\mathcal U}s(u;\mu),
    \label{eq:implicit_reference_score}
\end{equation}
where $\mathcal A$ is any smooth trajectory-based attribution score.
Thus $v(\mu)$ is the distribution selected using the reference path
generated by $\mu$.
Let the warmup mean be random with
$\mathbb E[\widehat\mu]=\mu$ and
$\operatorname{Cov}(\widehat\mu)=\Omega$.
Write $J=D_\mu v(\mu)$ for the sensitivity of the selected distribution
to the reference, and
$g=J^\top H(v(\mu)-q)$.

\begin{theorem}[Propagation of reference uncertainty]
\label{thm:pool_warmup_tradeoff}
Suppose $s$ is smooth and $v(\mu)$ is a unique regular maximizer.
For sufficiently concentrated warmup variation,
\begin{equation}
\begin{aligned}
    \frac12\mathbb E
    \|v(\widehat\mu)-v(\widehat\mu')\|_H^2
        &= \operatorname{tr}(HJ\Omega J^\top)
           + o(\|\Omega\|),\\
    \operatorname{Var}\!\left[D(v(\widehat\mu))\right]
        &= g^\top\Omega g+o(\|\Omega\|),
\end{aligned}
\label{eq:reference_uncertainty_law}
\end{equation}
where $\widehat\mu'$ is an independent warmup.
Moreover,
\begin{equation}
    D(v(q))-\mathbb E D(v(\widehat\mu))
    =
    D(v(q))-D(\mathbb E v(\widehat\mu))
    -
    \frac12\mathbb E
    \|v(\widehat\mu)-\mathbb E v(\widehat\mu)\|_H^2 .
    \label{eq:mean_dispersion_identity}
\end{equation}
\end{theorem}

\paragraph{Insight.}
Reference variation is transferred to selection through $J$.
Pool warmup can therefore exploit useful pool-specific information while
also introducing selection instability, whose magnitude depends jointly on
warmup heterogeneity and the attribution rule's reference sensitivity.
Only the target-relevant component of this variation propagates to
downstream risk. Equation~\eqref{eq:mean_dispersion_identity} further shows
that pool warmup may improve average performance, but any such gain is
offset by its induced dispersion. A fixed target reference removes the
pool-composition source of this uncertainty.

\subsection{Scoring Geometry in a Linearized Model}
\label{sec:theory:linear}

We next examine the geometric inductive biases of candidate scoring rules
along a fixed target trajectory. Continuous target gradient flow for
duration $\tau$ reaches endpoint $v_\tau$.
For intuition, we decompose any candidate displacement into a
target-aligned component and an orthogonal residual:
$u = av + b$, where $v = v_\tau \neq 0$, $a \in \mathbb{R}$,
and $b^\top H v = 0$.

\begin{theorem}[Geometric preferences of candidate scoring rules]
\label{thm:selection_risk}
Under an inactive normalization floor, score candidates over $[0, v]$ with $u = av + b \neq 0$.
\begin{enumerate}[leftmargin=*, itemsep=2pt, topsep=2pt]
    \item \emph{Normalized loss drop (TACS)}: The score $s(u) = 1 - \|u - v\|_H^2 / \|u\|_H^2$ attains a unique global maximum at $u^* = v$, achieving downstream risk $D(u^*) = \frac{1}{2}\|v - q\|_H^2$. For $a > 1/2$, the score strictly penalizes orthogonal variance: $\partial s(u) / \partial \|b\|_H^2 < 0$.
    \item \emph{Cosine alignment}: $c(u,v) = \langle u,v\rangle_H / (\|u\|_H \|v\|_H)$ is scale-invariant and maximized by any point along the ray $u = \alpha v$ ($\alpha > 0$).
    \item \emph{Dot-product alignment}: $d(u,v) = \langle u,v\rangle_H = a\|v\|_H^2$ scales linearly with $a$ and is invariant to the orthogonal component $b$. It induces the same candidate ranking as raw loss drop, since $R_u(0)-R_u(v)=d(u,v)-\frac12\|v\|_H^2$ for fixed $v$.
\end{enumerate}
\end{theorem}

\begin{figure}[t]
\centering
\includegraphics[width=0.90\linewidth]{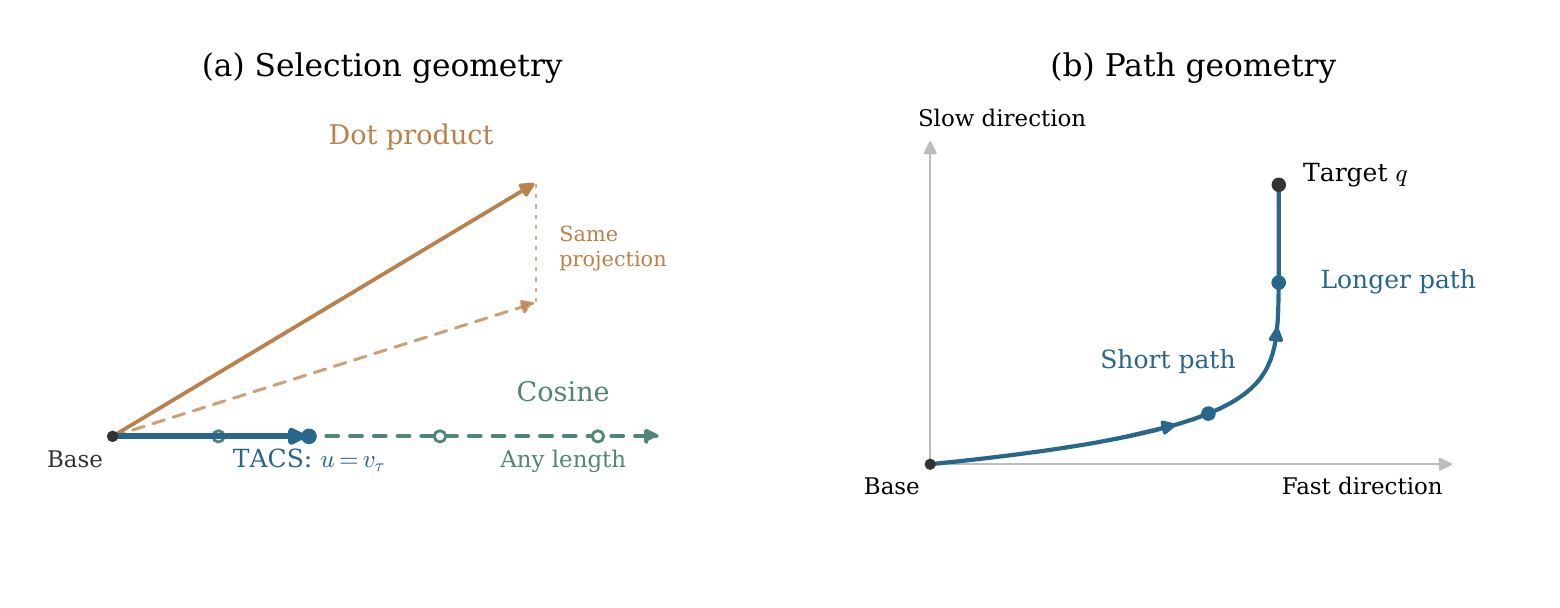}
\caption{\textbf{Geometric and spectral characteristics of the linearized model.}
\textbf{(a)} Scoring preferences in $H^{1/2}$ coordinates:
raw loss reduction is insensitive to orthogonal components;
cosine alignment leaves scale unconstrained;
normalized loss reduction uniquely favors $u=v_\tau$.
\textbf{(b)} Spectral mode evolution under target gradient flow: fast task modes ($\lambda_i \tau \gg 1$) contract rapidly, whereas slow modes are attenuated.}
\label{fig:theory_geometry}
\end{figure}

\begin{theorem}[Spectral bias of a finite target path]
\label{thm:target_modes}
Let $H e_i = \lambda_i e_i$ with
$\lambda_1 \ge \cdots \ge \lambda_d > 0$, and write
$q = \sum_i q_i e_i$. Then target gradient flow for duration $\tau$
reaches
\begin{equation}
    v_\tau
    =
    \sum_{i=1}^d
    \bigl(1-e^{-\lambda_i\tau}\bigr)q_i e_i.
\end{equation}
Thus larger-$\lambda_i$ target modes are learned earlier, while
smaller-$\lambda_i$ modes remain attenuated along short paths.
The relative target risk is
$D(v_\tau)/R_q(0)
=
\sum_i \lambda_i q_i^2 e^{-2\lambda_i\tau}
/\sum_i \lambda_i q_i^2$.
Moreover, letting
$C_q(k)=\sum_{i\le k}\lambda_iq_i^2/\|q\|_H^2$,
if two targets share $H$ and
$C_{q_A}(k)\ge C_{q_B}(k)$ for all $k$, then
$D(v_\tau^A)/R_{q_A}(0)
\le D(v_\tau^B)/R_{q_B}(0)$ for all $\tau>0$.
\end{theorem}

\paragraph{Insight.}
Within this population framework, Theorems~\ref{thm:selection_risk} and~\ref{thm:target_modes} formalize two complementary inductive biases illustrated in Figure~\ref{fig:theory_geometry}. First, \emph{direction and scale control}: raw loss reduction is
insensitive to the orthogonal component $b$, while cosine alignment
leaves candidate scale unconstrained. TACS uniquely favors
$u=v_\tau$ and, for fixed $a>1/2$, penalizes larger orthogonal
components. Its positive-score superlevel sets are bounded,
constraining candidate scale as well. Second, \emph{task-dependent mode learning}: finite target paths
capture fast modes while attenuating slow ones. Targets concentrated
in fast directions can therefore be captured by short paths,
whereas targets with substantial signal in slower directions require
deeper adaptation. This provides a population-level explanation
for task-dependent trajectory depth. In practice, held-out proxy
calibration selects the horizon.

\section{Empirical Study}\label{sec:empirical_study}
\ifarxiv We evaluate TACS across controlled synthetic and vision benchmarks, followed by large-scale instruction tuning on language models. Complete experimental settings appear in Appendix~\ref{appendix: exp_detail}. \fi

\subsection{Controlled Tasks: Mechanistic Validation}

\paragraph{Target-induced paths prioritize task-aligned parameter updates.}
We evaluate TACS on a controlled logistic mixture with known target and distractor components. As shown in Figure~\ref{fig:toy_logistic}a, the parameter displacement from retraining on a TACS-selected subset aligns substantially closer with the target-warmup path than with pool-warmed trajectories. This alignment yields consistently lower target test error across all selection budgets compared to ToV, LESS, and TracIn~\cite{pruthi2020estimatingtrainingdatainfluence} (Figure~\ref{fig:toy_logistic}b). This selective behavior is consistent with Theorem~\ref{thm:selection_risk}, where normalized loss reduction penalizes orthogonal distractor components, enabling TACS to prioritize target-aligned examples even in rare-target regimes.

\begin{figure}[t]
    \centering
    \includegraphics[width=0.738\textwidth]{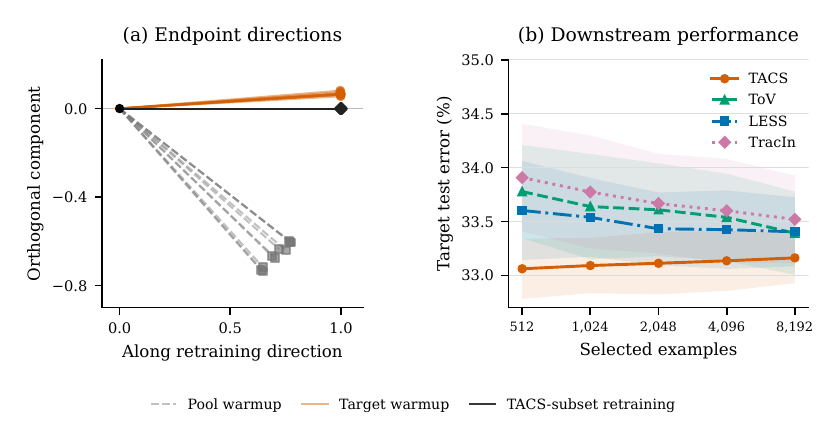}
    \caption{\textbf{Evaluation on a synthetic logistic mixture.} 
    \textbf{(a)} Retraining parameter displacements for TACS subsets compared with target and pool warmup directions. 
    \textbf{(b)} Downstream target test error across selection budgets (mean $\pm$ SD over ten seeds).}
    \label{fig:toy_logistic}
\end{figure}

\paragraph{Target trajectories filter label noise without pool contamination.}
We next examine selection robustness under distribution corruption using an ImageNet-pretrained \texttt{ResNet-18} on CIFAR-10 with 40\% symmetric label noise. Under severe noise, TACS substantially outperforms matched gradient baselines (LESS under full fine-tuning and TracIn under partial finetuning), while matching them on clean pools (Figure~\ref{fig:cv_noise_main}). Structurally, pool-warmed references incorporate updates from corrupted candidates, whereas TACS keeps candidate-pool variation out of reference construction. This reflects the mechanism in Theorem~\ref{thm:pool_warmup_tradeoff}: variation entering through a pool-dependent reference can propagate to both selection and downstream risk.
\begin{figure}[t]
\centering
\begin{subfigure}[t]{0.37\linewidth}
    \centering
    \includegraphics[width=0.90\linewidth]{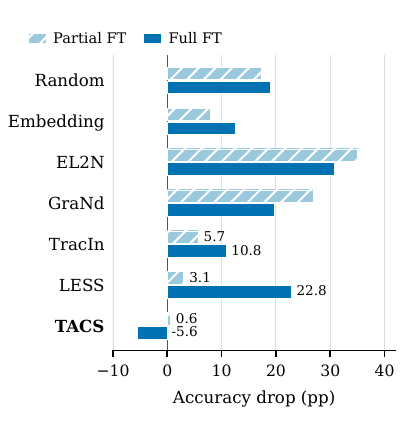}
\end{subfigure}
\hspace{0.04\linewidth}
\begin{subfigure}[t]{0.37\linewidth}
    \centering
    \includegraphics[width=0.90\linewidth]{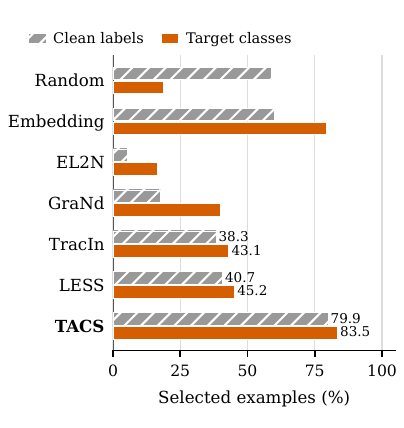}
\end{subfigure}
\caption{\textbf{Robustness under 40\% label noise on CIFAR-10 (ResNet-18).} 
\textbf{(left)} Downstream accuracy drop from clean to noisy candidate pools (lower is better). 
\textbf{(right)} Proportion of clean and target-class examples within selected subsets ($k=500$, three-seed means).}
\label{fig:cv_noise_main}
\end{figure}

\subsection{Instruction Tuning at Scale}
\label{sec:instruction_tuning}

We scale TACS to instruction tuning with \texttt{Llama-3.2-3B} across MMLU, TyDiQA, and BBH, selecting subsets from four instruction pools (COT, Dolly, Flan V2, and Oasst1) to retrain the base model. We compare TACS against random selection, representative dynamic selectors (LESS~\cite{xia2024lessselectinginfluentialdata} and ToV~\cite{jain2025trainvalidationtovfast}), and full-pool training.

\begin{table}[h!]
\centering
\small
\caption{\textbf{Downstream instruction-tuning performance across candidate pools and tasks.} 
Subset evaluation results (mean $\pm$ SD across three seeds; Llama-3.2-3B). Bold indicates the highest mean.}
\label{tab:main_results_llama3_wide}
\setlength{\tabcolsep}{3pt}
\renewcommand{\arraystretch}{1.0}
\begin{tabular}{@{}llccccc@{}}
\toprule
\textbf{Pool} & \textbf{Task} & \textbf{Random} & \textbf{LESS} & \textbf{ToV} & \textbf{TACS} & \textbf{Full pool} \\
\midrule
\multirow{3}{*}{\textbf{COT}} & MMLU & \makebox[2.7em][r]{54.28}\makebox[3.2em][l]{\scriptsize$\,\pm 0.23$} & \makebox[2.7em][r]{55.15}\makebox[3.2em][l]{\scriptsize$\,\pm 0.53$} & \makebox[2.7em][r]{54.94}\makebox[3.2em][l]{\scriptsize$\,\pm 0.71$} & \cellcolor{gray!10}\makebox[2.7em][r]{\textbf{55.57}}\makebox[3.2em][l]{\scriptsize$\,\pm 0.42$} & 53.26 \\
 & TyDiQA & \makebox[2.7em][r]{51.43}\makebox[3.2em][l]{\scriptsize$\,\pm 2.79$} & \makebox[2.7em][r]{\textbf{62.63}}\makebox[3.2em][l]{\scriptsize$\,\pm 2.60$} & \makebox[2.7em][r]{55.56}\makebox[3.2em][l]{\scriptsize$\,\pm 1.21$} & \cellcolor{gray!10}\makebox[2.7em][r]{62.12}\makebox[3.2em][l]{\scriptsize$\,\pm 1.41$} & 47.45 \\
 & BBH & \makebox[2.7em][r]{46.23}\makebox[3.2em][l]{\scriptsize$\,\pm 1.44$} & \makebox[2.7em][r]{\textbf{48.12}}\makebox[3.2em][l]{\scriptsize$\,\pm 0.37$} & \makebox[2.7em][r]{47.44}\makebox[3.2em][l]{\scriptsize$\,\pm 0.81$} & \cellcolor{gray!10}\makebox[2.7em][r]{46.11}\makebox[3.2em][l]{\scriptsize$\,\pm 1.18$} & 43.06 \\
\addlinespace[2pt]
\multirow{3}{*}{\textbf{Dolly}} & MMLU & \makebox[2.7em][r]{56.36}\makebox[3.2em][l]{\scriptsize$\,\pm 0.03$} & \makebox[2.7em][r]{\textbf{56.51}}\makebox[3.2em][l]{\scriptsize$\,\pm 0.14$} & \makebox[2.7em][r]{56.27}\makebox[3.2em][l]{\scriptsize$\,\pm 0.07$} & \cellcolor{gray!10}\makebox[2.7em][r]{56.39}\makebox[3.2em][l]{\scriptsize$\,\pm 0.12$} & 55.90 \\
 & TyDiQA & \makebox[2.7em][r]{50.60}\makebox[3.2em][l]{\scriptsize$\,\pm 0.88$} & \makebox[2.7em][r]{61.10}\makebox[3.2em][l]{\scriptsize$\,\pm 0.48$} & \makebox[2.7em][r]{60.15}\makebox[3.2em][l]{\scriptsize$\,\pm 0.93$} & \cellcolor{gray!10}\makebox[2.7em][r]{\textbf{63.21}}\makebox[3.2em][l]{\scriptsize$\,\pm 0.99$} & 56.95 \\
 & BBH & \makebox[2.7em][r]{48.15}\makebox[3.2em][l]{\scriptsize$\,\pm 0.24$} & \makebox[2.7em][r]{48.40}\makebox[3.2em][l]{\scriptsize$\,\pm 0.05$} & \makebox[2.7em][r]{47.72}\makebox[3.2em][l]{\scriptsize$\,\pm 0.59$} & \cellcolor{gray!10}\makebox[2.7em][r]{\textbf{48.46}}\makebox[3.2em][l]{\scriptsize$\,\pm 0.96$} & 48.24 \\
\addlinespace[2pt]
\multirow{3}{*}{\textbf{Flan V2}} & MMLU & \makebox[2.7em][r]{55.85}\makebox[3.2em][l]{\scriptsize$\,\pm 0.40$} & \makebox[2.7em][r]{54.93}\makebox[3.2em][l]{\scriptsize$\,\pm 0.63$} & \makebox[2.7em][r]{52.18}\makebox[3.2em][l]{\scriptsize$\,\pm 0.70$} & \cellcolor{gray!10}\makebox[2.7em][r]{\textbf{56.02}}\makebox[3.2em][l]{\scriptsize$\,\pm 0.19$} & 53.75 \\
 & TyDiQA & \makebox[2.7em][r]{57.47}\makebox[3.2em][l]{\scriptsize$\,\pm 1.92$} & \makebox[2.7em][r]{\textbf{59.46}}\makebox[3.2em][l]{\scriptsize$\,\pm 1.98$} & \makebox[2.7em][r]{58.59}\makebox[3.2em][l]{\scriptsize$\,\pm 0.69$} & \cellcolor{gray!10}\makebox[2.7em][r]{59.33}\makebox[3.2em][l]{\scriptsize$\,\pm 0.58$} & 62.38 \\
 & BBH & \makebox[2.7em][r]{47.50}\makebox[3.2em][l]{\scriptsize$\,\pm 0.61$} & \makebox[2.7em][r]{46.20}\makebox[3.2em][l]{\scriptsize$\,\pm 0.42$} & \makebox[2.7em][r]{47.72}\makebox[3.2em][l]{\scriptsize$\,\pm 1.13$} & \cellcolor{gray!10}\makebox[2.7em][r]{\textbf{48.18}}\makebox[3.2em][l]{\scriptsize$\,\pm 1.40$} & 49.07 \\
\addlinespace[2pt]
\multirow{3}{*}{\textbf{Oasst1}} & MMLU & \makebox[2.7em][r]{\textbf{56.37}}\makebox[3.2em][l]{\scriptsize$\,\pm 0.03$} & \makebox[2.7em][r]{56.16}\makebox[3.2em][l]{\scriptsize$\,\pm 0.57$} & \makebox[2.7em][r]{56.32}\makebox[3.2em][l]{\scriptsize$\,\pm 0.11$} & \cellcolor{gray!10}\makebox[2.7em][r]{\textbf{56.37}}\makebox[3.2em][l]{\scriptsize$\,\pm 0.06$} & 54.78 \\
 & TyDiQA & \makebox[2.7em][r]{32.78}\makebox[3.2em][l]{\scriptsize$\,\pm 2.93$} & \makebox[2.7em][r]{55.58}\makebox[3.2em][l]{\scriptsize$\,\pm 1.77$} & \makebox[2.7em][r]{55.52}\makebox[3.2em][l]{\scriptsize$\,\pm 0.68$} & \cellcolor{gray!10}\makebox[2.7em][r]{\textbf{58.92}}\makebox[3.2em][l]{\scriptsize$\,\pm 0.49$} & 41.09 \\
 & BBH & \makebox[2.7em][r]{47.72}\makebox[3.2em][l]{\scriptsize$\,\pm 1.02$} & \makebox[2.7em][r]{47.47}\makebox[3.2em][l]{\scriptsize$\,\pm 0.51$} & \makebox[2.7em][r]{48.58}\makebox[3.2em][l]{\scriptsize$\,\pm 0.56$} & \cellcolor{gray!10}\makebox[2.7em][r]{\textbf{48.67}}\makebox[3.2em][l]{\scriptsize$\,\pm 0.72$} & 48.98 \\
\addlinespace[2pt]
\midrule
\multirow{3}{*}{\textbf{Average}} & MMLU & \makebox[2.7em][r]{55.72}\makebox[3.2em][l]{\scriptsize$\,\pm 0.15$} & \makebox[2.7em][r]{55.69}\makebox[3.2em][l]{\scriptsize$\,\pm 0.25$} & \makebox[2.7em][r]{54.93}\makebox[3.2em][l]{\scriptsize$\,\pm 0.05$} & \cellcolor{gray!10}\makebox[2.7em][r]{\textbf{56.08}}\makebox[3.2em][l]{\scriptsize$\,\pm 0.10$} & 54.42 \\
 & TyDiQA & \makebox[2.7em][r]{48.07}\makebox[3.2em][l]{\scriptsize$\,\pm 0.03$} & \makebox[2.7em][r]{59.69}\makebox[3.2em][l]{\scriptsize$\,\pm 0.42$} & \makebox[2.7em][r]{57.45}\makebox[3.2em][l]{\scriptsize$\,\pm 0.13$} & \cellcolor{gray!10}\makebox[2.7em][r]{\textbf{60.89}}\makebox[3.2em][l]{\scriptsize$\,\pm 0.33$} & 51.97 \\
 & BBH & \makebox[2.7em][r]{47.40}\makebox[3.2em][l]{\scriptsize$\,\pm 0.50$} & \makebox[2.7em][r]{47.55}\makebox[3.2em][l]{\scriptsize$\,\pm 0.19$} & \makebox[2.7em][r]{\textbf{47.86}}\makebox[3.2em][l]{\scriptsize$\,\pm 0.42$} & \cellcolor{gray!10}\makebox[2.7em][r]{47.85}\makebox[3.2em][l]{\scriptsize$\,\pm 0.51$} & 47.34 \\
\midrule
\multicolumn{7}{l}{Base model (no training): MMLU 56.22; TyDiQA 42.00; BBH 48.61} \\
\bottomrule
\end{tabular}
\end{table}

\paragraph{Downstream performance across tasks and pools.}
Table~\ref{tab:main_results_llama3_wide} shows that TACS is competitive across the evaluated pools and tasks, with the clearest average gains on TyDiQA. Two factors may explain these gains: diverse pools offer more target-aligned examples to select from, while tasks like TyDiQA leave more room for improvement through adaptation. The latter interpretation is consistent with the task-dependent learning dynamics in Theorem~\ref{thm:target_modes}. 

\paragraph{Compute and storage efficiency via path amortization.}
Table~\ref{tab:empirical_compute_cost} highlights the efficiency of TACS across storage, compute, and scaling. Forward-only scoring eliminates candidate-gradient caching, reducing the scoring payload from over 35\,GB in LESS to under 7\,MB. More importantly, the target path is constructed once and reused across candidate pools, so adding a new pool requires only two forward evaluations per candidate. In our benchmark ($m=4,k=3$), this yields 53\% lower selection compute than ToV and 69\% lower than LESS. The complementary trade-off is that LESS amortizes more effectively across many targets for a fixed pool through reusable candidate gradients.

\begin{table}[!ht]
    \centering
    \footnotesize
    \renewcommand{\arraystretch}{1.15}
    \setlength{\tabcolsep}{6pt}
    \caption{\textbf{Stage-matched selection cost and scoring payload} on \texttt{Llama-3.2-3B} across $m=4$ pools and $k=3$ targets. GPU-hours encompass calibration, final reference construction, and candidate scoring under the Appendix~\ref{app:selection_cost_accounting} protocol. Downstream retraining is excluded.}
    \label{tab:empirical_compute_cost}
    \begin{tabular}{lccc}
    \toprule
    \textbf{Method} & \textbf{Cost Model (GPU h)} & \textbf{Total GPU h} & \textbf{Scoring Storage} \\
    \midrule
    LESS & $46.38m + 0.15mk$ & $187.35$ & $ 35480$\,MB \\
    ToV  & $1.06m + 10.04mk$ & $124.74$ & $ 26.0$\,MB \\
    TACS & $2.64k + 4.21mk$  & \textbf{$58.46$} (\textbf{$-69\%$ / $-53\%$}) & \textbf{$ 6.5$\,MB} \\
    \bottomrule
    \end{tabular}
\end{table}

\subsubsection{Ablation Studies}\label{sec:ablation_studies}

\paragraph{Low-capacity references can guide higher-capacity learners.}
Figure~\ref{fig:main_ablations} evaluates LoRA adapter rank during reference construction while downstream retraining uses rank-128 LoRA. Increasing reference capacity yields no downstream gains, with ranks 64 and 128 degrading performance on TyDiQA. Thus, effective targeted selection does not require matching the adaptation capacity of the downstream learner: a rank-one reference can identify useful data for substantially higher-capacity retraining. The resulting reference is also exceptionally lightweight, making fine-grained trajectory checkpointing practical even under modest storage budgets, where retaining many high-rank or full-model states can quickly become costly.

\paragraph{Trajectory depth exhibits task-dependent trends.} We find that useful trajectory depth is strongly task-dependent (Figure~\ref{fig:main_ablations}). While performance on MMLU and BBH remains stable across short horizons, TyDiQA benefits from deeper trajectories (improving up to 64 proxy epochs). This dependence persists with capacity-matched target warmup and unnormalized loss scoring (Appendix~\ref{app:extended_related_work}). As characterized by Theorem 4.3, deeper target paths progressively expose slower target modes. This cross-task divergence establishes trajectory depth as an important design variable in target-reference construction, motivating the task-adaptive calibration used by TACS.

\begin{figure}[!ht]
    \centering
    \begin{subfigure}[t]{0.36\linewidth}
        \centering
        \includegraphics[width=0.90\linewidth]{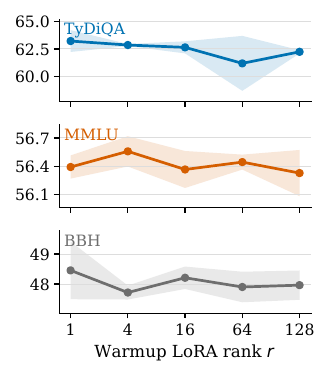}
    \end{subfigure}\hspace{0.04\linewidth}
    \begin{subfigure}[t]{0.36\linewidth}
        \centering
        \includegraphics[width=0.90\linewidth]{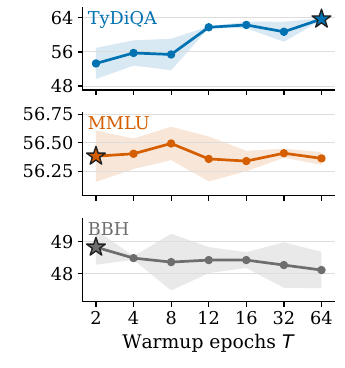}
    \end{subfigure}
    \caption{\textbf{Ablation of reference capacity and trajectory horizon on Dolly.} 
    Downstream performance (three-seed mean $\pm$ SD) as warmup LoRA rank (\textbf{left}) and proxy-epoch depth $T$ (\textbf{right}) vary. Stars denote calibration-selected depths.}
    \label{fig:main_ablations}
\end{figure}

\section{Related Work}\label{sec:related_work}
Targeted data selection identifies training examples aligned with downstream
requirements, typically through bilevel optimization~\cite{killamsetty2021glistergeneralizationbaseddata},
gradient matching~\cite{killamsetty2021gradmatchgradientmatchingbased}, or
dynamic attribution signals computed along training paths~\cite{xia2024lessselectinginfluentialdata,
jain2025trainvalidationtovfast,min2026gisttargeteddataselection}. More broadly,
instance attribution methods estimate training example influence through
gradients, inverse-Hessian approximations, or parameter perturbations during
adaptation~\cite{pmlr-v70-koh17a,pruthi2020estimatingtrainingdatainfluence,
park2023trakattributingmodelbehavior}.

Separately, conditional loss-reduction methods score candidates by contrasting
losses against an adapted or holdout model~\cite{moore2010intelligent,
pmlr-v162-mindermann22a,brandfonbrener2024colorfilter}. The unnormalized TACS
score shares this loss-drop form. These methods motivate loss reduction as a
selection signal, while TACS studies the reference states that define it. When
candidate pools differ from the target and supervision is limited to a compact
proxy, TACS controls reference capacity and calibrates trajectory depth to the
target task. Appendix~\ref{app:extended_related_work} provides detailed
comparisons and additional depth ablations.

\section{Conclusion.}
We demonstrated that reference paths in trajectory-based selection serve as active selection priors. To eliminate pool dependence, TACS grounds reference checkpoints in a compact target proxy and scores candidates via forward-only normalized loss reduction. Across diverse benchmarks, TACS achieves competitive downstream utility while drastically reducing selection compute and storage.

By establishing reference construction as an explicit design choice, our findings advance principled, resource-efficient data curation.

\section*{Acknowledgments}
{\emergencystretch=2em
We thank Yue Xing for helpful discussions. This research used computational resources provided through the National Artificial Intelligence Research Resource (NAIRR) Pilot under awards NAIRR260246 and NAIRR250187, and the DeltaAI advanced computing and data resource, which is supported by the National Science Foundation under award NSF-OAC 2320345.\par}

\section*{Impact Statement}\label{sec:broader_impacts}
Targeted data selection substantially reduces the computational overhead and environmental footprint associated with fine-tuning large-scale foundation models by identifying high-utility candidate subsets. Furthermore, scoring candidates against explicit reference checkpoints provides an interpretable audit trail of data attribution for specific downstream behaviors. 

However, because the selection trajectory is guided directly by the target validation proxy \(\Zval\), the curated data distribution necessarily reflects the inductive biases and distributional properties of this proxy. If \(\Zval\) encodes societal biases, skewed demographic representations, or toxic content, the selection mechanism may preferentially retain candidates that reinforce these undesirable traits. Moreover, the efficiency gains offered by reusable reference paths could lower the technical barrier for adapting models toward malicious or dual-use tasks. Responsible deployment therefore requires rigorous provenance documentation for validation proxies and thorough data auditing of selected subsets prior to full downstream retraining.

\bibliography{shared/references}
\appendix
\newpage
\section{Limitations}\label{sec:limitations}

\paragraph{Dependence on the validation proxy.}
TACS guides its reference trajectory using the validation proxy \(\Zval\); consequently, candidate scoring reflects the task aspects captured within this proxy. If a proxy is unrepresentative or omits key facets of the target distribution, candidates addressing the missing behaviors may receive suboptimal scores. Effective proxy curation therefore remains a natural complement to target-aligned selection.

\paragraph{Capacity regularization and calibration.}
Restricting the reference path to a low-rank LoRA subspace acts as an explicit capacity constraint, moderating rapid memorization of small proxies. However, adaptation duration and step size must still be calibrated to prevent proxy-specific over-adaptation. While cross-validation (Appendix~\ref{sec:appendix_hyperparameters}) provides an automated selection rule, calibration over very compact proxies (such as the nine-example TyDiQA proxy) involves few held-out instances per fold, which can increase sensitivity to proxy partition choices.

\paragraph{Scope of the theoretical analysis.}
The theoretical derivations in Section~\ref{sec:theory} analyze candidate preferences and reference perturbations under continuous gradient flow and linearized population dynamics. In practice, the implementation uses discrete AdamW updates within a low-rank parameterization. Furthermore, our analysis characterizes the instantaneous preference rule for individual candidate responses to a finite warmup endpoint; it does not model post-selection retraining dynamics or guarantee the subset that globally minimizes downstream target risk.

\paragraph{Amortization trade-offs across pools and targets.}
TACS offers substantial efficiency when a fixed target trajectory is amortized to screen multiple candidate pools. However, this efficiency profile introduces an operational tradeoff: adapting to each new downstream target requires constructing a new reference path, whereas gradient-caching selectors (such as LESS) can amortize precomputed candidate gradients across multiple target proxies. TACS is therefore most advantageous when curating diverse candidate sources for a focused set of downstream tasks.

\section{Extended Related Work}
\label{app:extended_related_work}

\begin{figure}[ht]
\centering
\includegraphics[width=\textwidth]{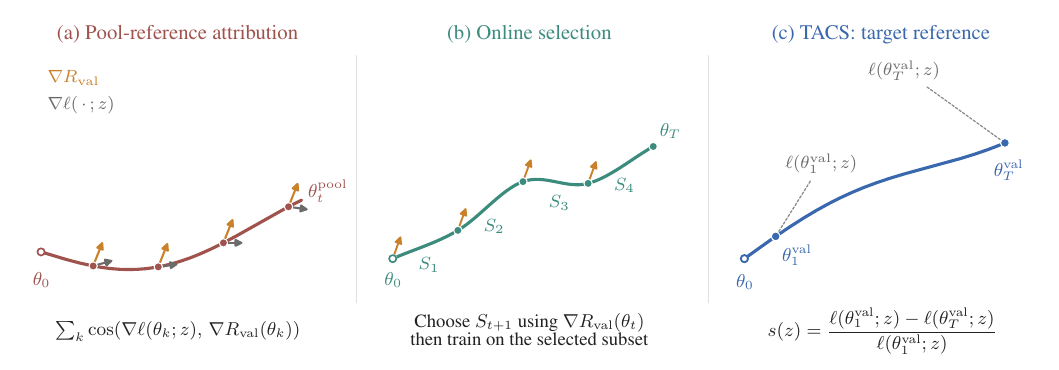}
\caption{\textbf{Taxonomy of training dynamics in targeted data selection.} \textbf{(a)} Pool-reference attribution warms up on candidate data and measures gradient alignment with the target proxy across checkpoints. \textbf{(b)} Online selection iteratively alternates subset filtering and model retraining. \textbf{(c)} TACS constructs an early, regularized target-proxy trajectory and evaluates candidates via normalized loss response along this path. Paths are illustrative.}
\label{fig:method_comparison}
\end{figure}

\paragraph{Pool-reference attribution.}
A prominent family of dynamic selectors evaluates candidate utility at parameter states visited during a candidate-pool warmup. LESS~\cite{xia2024lessselectinginfluentialdata} calculates projected cosine gradient similarities across intermediate checkpoints of a pool-trained model. ToV~\cite{jain2025trainvalidationtovfast} evaluates candidates via empirical loss improvements induced by one-step target perturbations around pool checkpoints, while GIST~\cite{min2026gisttargeteddataselection} projects gradients into a low-rank target subspace. Although varying in their local scoring metrics, these methods share a foundational dependency: the reference parameter states where candidates are judged are fundamentally dictated by the candidate pool itself. As established in Section~\ref{sec:motivation:observations} and Proposition~\ref{prop:reference_errors}, reference construction can couple candidate
scoring to pool statistics. TACS departs fundamentally from this paradigm by constructing the reference trajectory exclusively from the target proxy (Figure~\ref{fig:method_comparison}c), entirely decoupling reference states from candidate data.

\paragraph{Online selection and coreset optimization.}
Coreset and online selection algorithms optimize subsets during the actual training process. CRAIG~\cite{mirzasoleiman2020coresetsdataefficienttrainingmachine}, GradMatch~\cite{killamsetty2021gradmatchgradientmatchingbased}, and GLISTER~\cite{killamsetty2021glistergeneralizationbaseddata} iteratively match batch gradients to full-pool or validation objectives. Dynamic bilevel selectors~\cite{wang2026twostageoptimizerawareonlinedata,wang2026opusefficientprincipleddata,hao2026blisslightweightbilevelinfluence} continuously re-estimate candidate utility as training proceeds, and GradAlign~\cite{yang2026gradaligngradientaligneddataselection} enforces directional alignment during reinforcement learning fine-tuning. While effective, online methods incur substantial retraining overhead because data selection is tightly interleaved with model optimization. In contrast, TACS performs a single, forward-only selection pass from a static target trajectory prior to retraining, decoupling selection from downstream model execution.

\paragraph{Loss-based filtering versus trajectory-guided attribution.}
Static loss heuristics identify noisy or outlier instances~\cite{swayamdipta2020dataset,toneva2019empiricalstudyexampleforgetting,pleiss2020identifyingmislabeleddatausing,paul2023deeplearningdatadiet} without downstream task awareness. In contrast, conditional loss-difference methods evaluate candidate utility via an adapted or holdout model: Moore--Lewis selection~\cite{moore2010intelligent} contrasts domain-specific against general language models, RHO-LOSS~\cite{pmlr-v162-mindermann22a} uses an independently trained holdout model, and CoLoR-Filter~\cite{brandfonbrener2024colorfilter} measures the unnormalized loss drop from downstream adaptation. Motivated by this broad principle, TACS connects closely to conditional loss filtering: without scale normalization, its endpoint score reduces to an unnormalized target-induced loss drop.

\begin{figure}[H]
\centering
\includegraphics[width=0.98\textwidth]{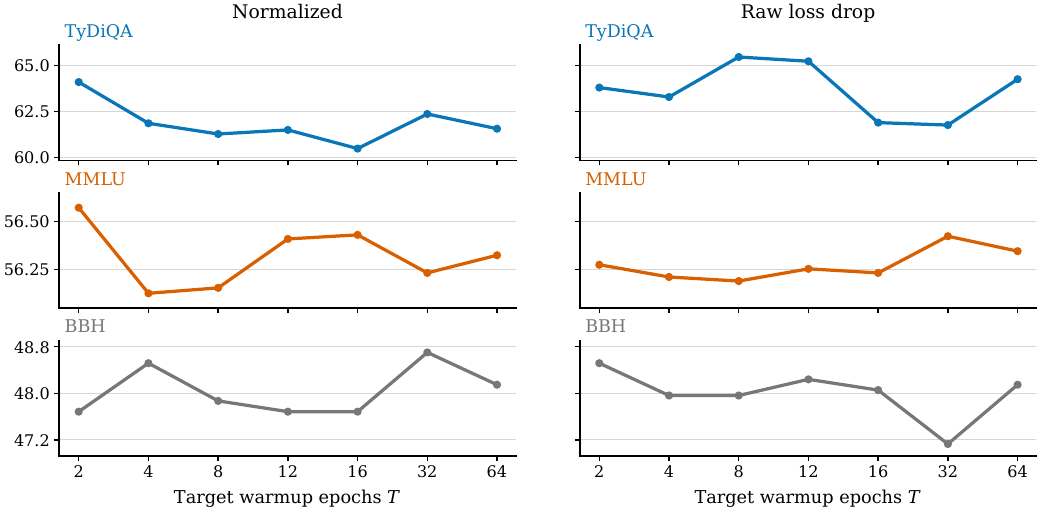}
\caption{\textbf{Reference-depth sensitivity with capacity-matched warmup.} Target warmup uses the same model and LoRA configuration as downstream retraining. Curves show downstream performance of subsets selected using normalized (left) and raw (right) endpoint loss reduction across target warmup depths.}
\label{fig:capacity_matched_depth}
\end{figure}

Our targeted instruction-tuning setting uses compact target proxies to construct reference paths. Figure~\ref{fig:capacity_matched_depth} examines capacity-matched target warmup under normalized and raw loss reduction. Performance varies non-monotonically with depth, and the observed best depth differs across tasks under both scoring rules. These observations motivate task-specific depth calibration. Together with the reference-capacity ablation in Figure~\ref{fig:main_ablations}, they support capacity control and horizon calibration when constructing reference paths from limited target data.

\section{Notation Table}
\label{app:notation}

\begin{table}[htbp]
\centering
\small
\caption{Notation used throughout the paper. We reserve calligraphic \(\Z\) for finite datasets or pools and calligraphic \(\D\) for population distributions.}
\label{tab:notation}
\setlength{\tabcolsep}{6pt}
\begin{tabular}{@{}ll@{}}
\toprule
\textbf{Symbol} & \textbf{Meaning} \\
\midrule
\(\Zpool\) or \(\Z\) & finite candidate pool from which examples are selected \\
\(\Zval\) & finite target validation/proxy set used to construct the warmup trajectory \\
\(\D\), \(\D_{\mathrm{tar}}\) & target population distribution underlying \(\Zval\) and test data \\
\(z\) & one candidate or validation example; for SFT, an input-output sequence pair \\
\(S \subseteq \Zpool\) & selected subset used for downstream retraining \\
\(N\) & selection budget, i.e., \(|S|=N\) \\
\(\theta_0, \theta_t, \theta_T\) & initial, intermediate, and final model parameters along a trajectory \\
\(\theta_t^{\val}\), \(\theta_t^{\mathrm{pool}}\) & validation-induced and pool-induced trajectory states \\
\(\mathcal{T}=\{\theta_t\}_{t=0}^T\) & optimization trajectory used for attribution or scoring \\
\(\eta_t\) & learning rate or continuous-time step size at time \(t\) \\
\(\ell(\theta;z)\) & per-example loss of parameters \(\theta\) on example \(z\) \\
\(R(\theta)\), \(R_{\val}(\theta)\), \(R_S(\theta)\) & target, validation empirical, and selected-subset risks \\
\(F(z,\mathcal{T})\) & path alignment of candidate \(z\) along \(\mathcal{T}\) \\
\(\bar g\), \(\Sigma\) & mean and covariance of candidate gradients at a reference state \\
\(\delta_v(z)\), \(\phi(z)\) & score change under reference error \(v\); pool typicality \\
\(\Delta \ell(z)\) & endpoint loss drop \(\ell(\theta_1^{\val};z)-\ell(\theta_T^{\val};z)\) \\
\(s(z)\) & normalized TACS score, \(\Delta \ell(z)/\max\{\ell(\theta_1^{\val};z),\varepsilon\}\) \\
\(r,\alpha\) & LoRA rank and scaling factor used in adapter warmups \\
\(M,K\) & number of cross-validation folds and trajectory-depth candidates \\
\bottomrule
\end{tabular}
\end{table}

\section{Trajectory Scores as Path Integrals and Reference Dynamics}
\label{app:proof}
\label{app:inf}

This appendix provides the formal continuous-time foundations for the trajectory-alignment formulation introduced in Section~\ref{sec:motivation}, details the perturbation analysis of reference paths, and connects finite-time trajectory attribution to classical influence functions.

\subsection{Trajectory-Based Attribution under Continuous Flow}
\label{app:inf:tracin}

Let \(R_0\) denote the reference objective governing model optimization, and let \(R\) denote the target validation objective. Consider the learning-rate-weighted gradient flow trajectory:
\begin{equation}
\label{eq:R0_flow_appendix}
    \frac{d\theta_t}{dt} = -\eta_t \nabla R_0(\theta_t), \qquad \theta_0 \text{ fixed},
\end{equation}
where \(\eta_t \ge 0\). For any candidate example \(z\), we define the finite-time trajectory attribution score over horizon \(T\) as:
\begin{equation}
\label{eq:path_score_appendix}
    \mathcal{I}_{R,R_0}^{T}(z) := \int_0^T \eta_t \left\langle \nabla R(\theta_t), \nabla_\theta \ell(\theta_t;z) \right\rangle \,dt .
\end{equation}

\begin{lemma}[Convergence of discrete checkpoint attribution]
\label{lem:discrete_continuous}
Let checkpoints follow discrete gradient updates \(\theta_{k+1} = \theta_k - (\tau_{k+1}-\tau_k)\eta(\tau_k) \nabla R_0(\theta_k)\) across a partition \(0 = \tau_0 < \tau_1 < \dots < \tau_K = T\). Suppose \(\nabla R_0\), \(\nabla R\), and \(\nabla_\theta\ell(\cdot;z)\) are Lipschitz continuous and bounded on a compact neighborhood of the trajectory, and \(\eta_t\) is continuous on \([0,T]\). As the mesh size \(h = \max_k(\tau_{k+1}-\tau_k) \to 0\), the discrete weighted attribution sum converges to \(\mathcal{I}_{R,R_0}^{T}(z)\):
\begin{equation}
    \lim_{h \to 0} \sum_{k=0}^{K-1} (\tau_{k+1}-\tau_k)\,\eta(\tau_k) \left\langle \nabla R(\theta_k), \nabla_\theta \ell(\theta_k;z) \right\rangle = \mathcal{I}_{R,R_0}^{T}(z).
\end{equation}
\end{lemma}
\begin{proof}
Under the stated Lipschitz and boundedness conditions, explicit Euler discretization of Eq.~\eqref{eq:R0_flow_appendix} guarantees uniform trajectory error \(\max_k \|\theta_k - \theta(\tau_k)\| = \mathcal{O}(h)\)~\citep{hairer1993solving}. Define the instantaneous alignment function \(\psi(\theta) = \langle \nabla R(\theta), \nabla_\theta\ell(\theta;z) \rangle\). Because \(\nabla R\) and \(\nabla_\theta\ell\) are Lipschitz and bounded, \(\psi(\theta)\) is Lipschitz on the trajectory domain. Consequently, substituting continuous states \(\theta(\tau_k)\) for discrete checkpoints \(\theta_k\) shifts the weighted sum by \(\mathcal{O}(h)\). The resulting quantity \(\sum_k (\tau_{k+1}-\tau_k)\eta(\tau_k)\psi(\theta(\tau_k))\) is a standard Riemann sum of the continuous integrand \(\eta_t\psi(\theta_t)\), which converges directly to the Riemann integral \(\mathcal{I}_{R,R_0}^T(z)\) as \(h \to 0\).
\end{proof}

Equation~\eqref{eq:path_score_appendix} cleanly decouples two roles: the target objective \(R\) provides the directional gradient field, whereas the reference objective \(R_0\) determines the trajectory manifold over which this field is evaluated.

\subsection{Reference-Path Dependence}
\label{app:inf:reference}

\begin{definition}[Path-induced attribution gap]
\label{def:path_gap}
Let \(R_a\) and \(R_b\) be two distinct reference objectives generating trajectories \(\theta_t^{(a)}\) and \(\theta_t^{(b)}\), respectively. For candidate \(z\), the path-induced attribution gap is:
\begin{equation}
\label{eq:path_gap_appendix}
    \Gamma_T(z;R,R_a,R_b) := \mathcal{I}_{R,R_a}^{T}(z) - \mathcal{I}_{R,R_b}^{T}(z).
\end{equation}
\end{definition}

When candidate-pool and target distributions diverge, a pool-induced trajectory (\(R_0 = R_{\mathrm{pool}}\)) and a target-induced trajectory (\(R_0 = R_{\val}\)) traverse divergent parameter submanifolds. Consequently, the target gradient field \(\nabla R(\theta_t)\) is evaluated at fundamentally different coordinates, altering both the magnitude and relative ordering of candidate scores. TACS resolves this divergence by fixing \(R_0 = R_{\val}\).

\subsection{First-Order Reference Errors and Directional Bias}
\label{app:inf:first_order}

This subsection proves Proposition~\ref{prop:reference_errors}, examines common-displacement properties, and analyzes the corresponding bias in gradient-alignment metrics.

\begin{proposition}[First-order reference errors and pool typicality]
\label{prop:reference_errors}
Let $s(z;\theta)$ be differentiable at the reference state $\theta$, and set $a_z=\nabla_\theta s(z;\theta)$. A small reference displacement $v$ changes the score to first order by $\delta_v(z)=\langle a_z,v\rangle$. For any candidate statistic $\phi(z)$,
$\operatorname{Cov}_z(\delta_v,\phi)=\langle v,\operatorname{Cov}_z(a_z,\phi)\rangle$.
If $v$ is independent of the candidate pool and has mean zero, the expected covariance is zero. For a raw endpoint loss drop perturbed at its final checkpoint, let $g_z$ be the candidate gradient, $\bar g=\mathbb E_z g_z$, $\Sigma=\operatorname{Cov}_z(g_z)$, and $\phi(z)=\langle g_z,\bar g\rangle$. A pool-descent displacement $v=-\lambda\bar g$ with $\lambda>0$ then gives $\operatorname{Cov}_z(\delta_v,\phi)=\lambda\bar g^\top\Sigma\bar g\geq0$.
\end{proposition}

\begin{proof}[Proof of Proposition~\ref{prop:reference_errors}]
Let \(s(z;\theta)\) denote a candidate scoring function evaluated at reference state \(\theta\). A first-order perturbation \(v\) shifts the score by \(\delta_v(z) = \langle a_z, v \rangle\), where \(a_z = \nabla_\theta s(z;\theta)\) is the candidate response vector. For any candidate statistic \(\phi(z)\), bilinearity of covariance yields:
\begin{equation}
\label{eq:first_order_cov_appendix}
    \operatorname{Cov}_z\big(\delta_v(z),\phi(z)\big) = \langle v, c_\phi \rangle, \qquad \text{where } c_\phi = \operatorname{Cov}_z\big(a_z,\phi(z)\big).
\end{equation}
The vector \(c_\phi\) is fixed by the candidate pool and the unperturbed state. If perturbation \(v\) is statistically independent of the candidate pool with \(\mathbb{E}[v] = 0\), taking the expectation over \(v\) yields:
\begin{equation}
    \mathbb{E}_v\left[\operatorname{Cov}_z(\delta_v,\phi)\right] = \langle \mathbb{E}[v], c_\phi \rangle = 0.
\end{equation}
For the raw endpoint loss difference \(s(z) = \ell(\theta_a;z) - \ell(\theta_b;z)\), a deviation at endpoint \(\theta_b\) gives response vector \(a_z = -\nabla_\theta \ell(\theta_b;z) =: -g_z\). Evaluating the covariance against candidate pool typicality \(\phi(z) = \langle g_z, \bar g \rangle\) with pool covariance \(\Sigma = \operatorname{Cov}_z(g_z)\) gives \(c_\phi = -\Sigma\bar g\). Thus:
\begin{equation}
    \operatorname{Cov}_z(\delta_v,\phi) = -v^\top\Sigma\bar g.
\end{equation}
Under a pool-descent step \(v = -\lambda\bar g\) (\(\lambda > 0\)), this evaluates to:
\begin{equation}
    \operatorname{Cov}_z(\delta_v,\phi) = \lambda\,\bar g^\top\Sigma\bar g \ge 0,
\end{equation}
which is non-negative because covariance matrix \(\Sigma\) is positive semi-definite.
\end{proof}

\paragraph{Joint endpoint perturbations.}
For the raw loss drop \(s(z) = \ell(\theta_a;z) - \ell(\theta_b;z)\), if both reference checkpoints undergo an identical displacement \(v\), the first-order variation is:
\begin{equation}
    \delta_v(z) = \left[\nabla_\theta\ell(\theta_a;z) - \nabla_\theta\ell(\theta_b;z)\right]^\top v.
\end{equation}
By the mean value theorem, this can be written as \(\delta_v(z) = (\theta_a - \theta_b)^\top \bar H_z v\), where \(\bar H_z = \int_0^1 \nabla_\theta^2\ell((1-t)\theta_b + t\theta_a;z)\,dt\) represents the path-averaged candidate Hessian. Thus, common displacements do not vanish identically in general non-convex settings; rather, they scale with candidate curvature and the trajectory displacement \(\|\theta_a - \theta_b\|\).

\paragraph{Bias in gradient-alignment selectors.}
Consider local gradient-alignment metrics of the form \(s(z;\theta) = \langle\nabla R_{\val}(\theta), \nabla\ell(\theta;z)\rangle\). Differentiating with respect to \(\theta\) yields the response vector \(a_z = H_{\val}(\theta)g_z + H_z(\theta)\nabla R_{\val}(\theta)\), where \(H_{\val}\) and \(H_z\) denote the validation and candidate Hessians. Under locally homogeneous candidate curvature ($H_z\approx\bar H$),
the second term is approximately common across candidates.
Hence $c_\phi\approx H_{\mathrm{val}}\Sigma\bar g$,
and a pool-descent perturbation $v=-\lambda\bar g$ gives
\begin{equation}
\operatorname{Cov}_z(\delta_v,\phi)
\approx
-\lambda\,\bar g^\top H_{\mathrm{val}}\Sigma\bar g.
\end{equation}
This couples score changes to the pool statistics $\bar g,\Sigma$
and target curvature $H_{\mathrm{val}}$, with the covariance sign
determined by their joint geometry.

\paragraph{Empirical endpoint sensitivity diagnostic.}
Figure~\ref{fig:reference_perturbation} evaluates the sensitivity of TACS rankings to imposed reference displacements: \(\theta_T \leftarrow \theta_T + m\,\|\theta_T - \theta_1\|\,u\), where \(u\) is a unit vector and \(m \in \{0.25, 0.5, 1, 2, 4\}\). Across model--task settings, structured displacements along pool-gradient or validation-gradient directions induce substantially larger ranking shifts than isotropic perturbations, corroborating the directional sensitivity characterized by Proposition~\ref{prop:reference_errors}.

\begin{figure}[t]
\centering
\includegraphics[width=0.95\textwidth]{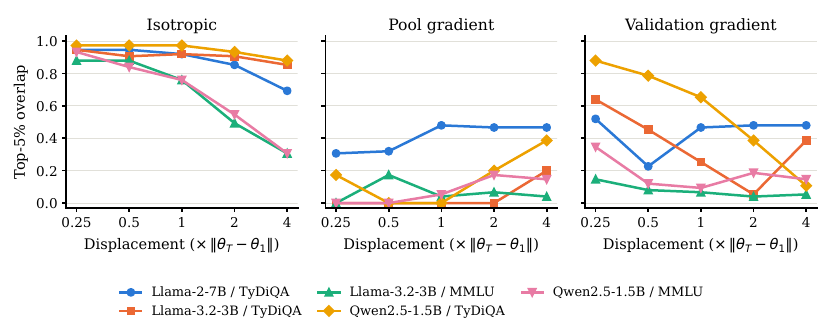}
\caption{\textbf{Ranking sensitivity under structured reference perturbations.} Selection overlap (top-5\%) with the unperturbed TACS subset as the reference endpoint is displaced along isotropic, pool-gradient, and validation-gradient directions across varying perturbation scales (1,500 Dolly candidates).}
\label{fig:reference_perturbation}
\end{figure}

\subsection{Extension to Adaptive Optimizers}
\label{app:inf:adaptive}

The trajectory-conditioning effect extends directly to adaptive optimizers such as AdamW. In the small-step limit, Adam induces a preconditioned gradient flow:
\begin{equation}
    \frac{d\theta_t}{dt} = -\eta_t P_t^{-1} \widetilde{\nabla} R_0(\theta_t),
\end{equation}
where \(P_t\) incorporates running second-moment scaling and \(\widetilde{\nabla}R_0\) reflects momentum smoothing. The corresponding trajectory attribution score becomes:
\begin{equation}
\label{eq:adaptive_score_appendix}
    \mathcal{I}_{R,R_0,\mathrm{opt}}^{T}(z) := \int_0^T \eta_t \left\langle \nabla R(\theta_t), P_t^{-1} \widetilde{g}(z,\theta_t) \right\rangle \,dt,
\end{equation}
where \(\widetilde{g}(z,\theta_t)\) represents the momentum-smoothed candidate gradient. While optimizer preconditioning rescales the local metric tensor, the evaluation points \(\theta_t\) remain strictly governed by reference objective \(R_0\).

\subsection{Connection to Classical Influence Functions}
\label{app:inf:classical}

For a sample-weighted empirical objective \(L(\theta;w) = \sum_j w_j \ell(\theta;z_j)\), differentiating continuous gradient flow with respect to weighting \(w_i\) yields the transported parameter sensitivity:
\begin{equation}
    \frac{\partial\theta_T}{\partial w_i} = -\int_0^T \Phi(T,s)\,\eta_s \nabla_\theta\ell(\theta_s;z_i)\,ds,
\end{equation}
where the transition matrix \(\Phi(T,s) = \mathcal{T}\exp\left(-\int_s^T \eta_u \nabla_\theta^2 L(\theta_u)\,du\right)\) transports earlier perturbations through subsequent trajectory Hessians. 

Trajectory-based attribution differs from classical influence in two operational respects:
\begin{enumerate}[leftmargin=*, itemsep=1pt, topsep=1pt]
    \item \emph{Omission of recursive Hessian transport}: Trajectory attribution measures instantaneous gradient alignment at each visited state, bypassing the computationally prohibitive computation of \(\Phi(T,s)\).
    \item \emph{Finite-time transient regime}: Classical influence functions evaluate asymptotic sensitivities at an isolated local minimum \(\theta^*\), yielding \(-H(\theta^*)^{-1}\nabla_\theta\ell(\theta^*;z_i)\). In contrast, modern targeted adaptation operates within early, non-converged parameter regimes where the geometric structure of the reference path itself constitutes an active component of data selection.
\end{enumerate}

\section{Experimental Details}
\label{appendix: exp_detail}
\label{app:additional_results}

\subsection{Controlled Logistic Regression Details}
\label{app:logreg}

To isolate candidate selection geometry under exact distributions, we instantiate synthetic logistic models:
\begin{equation}
    x \sim \mathcal{N}(0, I_d), \qquad \Pr(y=1 \mid x) = \sigma(\langle x, \theta^* \rangle), \qquad \sigma(t) = \frac{1}{1 + \exp(-t)}.
\end{equation}
Trajectories use full-batch gradient descent with linear learning-rate decay, retraining selected subsets from the shared initialization. Configurations are summarized in Table~\ref{tab:logreg_settings}, with complete budget curves shown in Figure~\ref{fig:logreg_appendix_summary}.

\begin{table}[ht]
\centering
\small
\caption{\textbf{Controlled logistic regression configurations.} The balanced pool is an equal mixture of target and orthogonal-distractor environments; the rare-target pool contains 5\% target data and six distractors at $80^\circ$ to the target.}
\label{tab:logreg_settings}
\setlength{\tabcolsep}{6pt}
\begin{tabular}{@{}lcc@{}}
\toprule
\textbf{Configuration} & \textbf{Balanced Mixture} & \textbf{Rare-Target Stress Test} \\
\midrule
Dimension $d$ & 10 & 48 \\
Pool / Validation / Test sizes & 131,072 / 1,024 / 10,000 & 16,000 / 8 / 2,000 \\
Pool reference schedule & 4,096 rows, 4 steps, LR 0.5 & 2,000 rows, 36 steps, LR 0.55 \\
Selection budgets & 128--8,192 ($\times\sqrt{2}$) & 50, 100, 200, 400 \\
Retraining schedule & 120 steps, LR 0.7 & 120 steps, LR 0.7 \\
Seeds & 10 & 10 \\
\bottomrule
\end{tabular}
\end{table}

\paragraph{Endpoint projection coordinates.}
Figure~\ref{fig:toy_logistic}a projects retraining ($\theta_T^{(r)}-\theta_0^{(r)}$), validation-warmup ($\theta_T^{(v)}-\theta_0^{(v)}$), and pool-warmup displacements onto an orthonormal basis defined by the retraining direction $e_1 = (\theta^{(r)}_T-\theta^{(r)}_0)/\|\theta^{(r)}_T-\theta^{(r)}_0\|_2$ and the residual orthogonal component $e_2 = u / \|u\|_2$, where $u = \theta^{(v)}_T-\theta^{(v)}_0-\langle\theta^{(v)}_T-\theta^{(v)}_0,e_1\rangle e_1$.

\begin{figure}[t]
\centering
\includegraphics[width=0.95\textwidth]{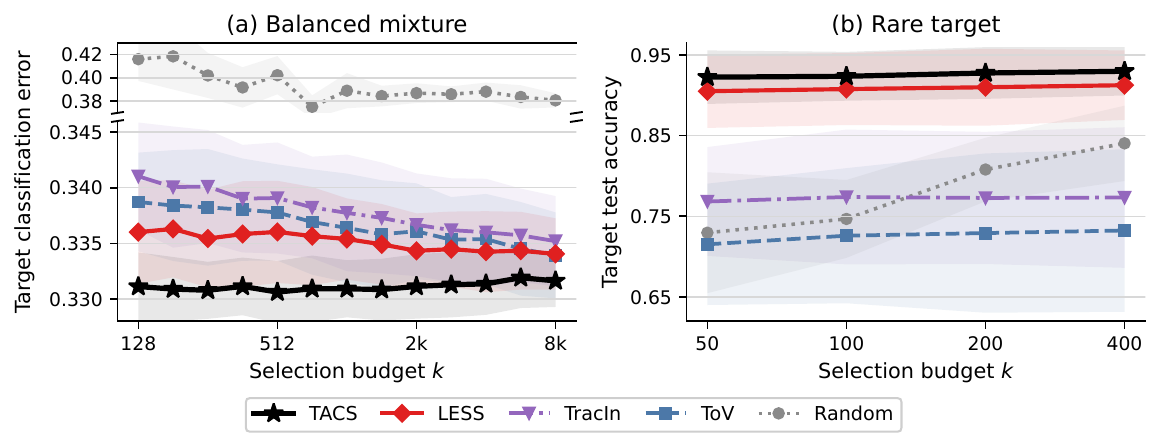}
\caption{\textbf{Controlled logistic regression metrics across budgets} (mean $\pm$ std over 10 seeds). \textbf{(a)} Balanced mixture target error across budgets. \textbf{(b)} Rare-target test accuracy.}
\label{fig:logreg_appendix_summary}
\end{figure}

\paragraph{Scoring protocol.}
TACS scores candidates by normalized loss reduction between initial and endpoint checkpoints along the validation path: $s_i^{\rm TACS} = [\ell_i(\theta_{\rm ref}) - \ell_i(\theta_{\rm last})] / \max\{\ell_i(\theta_{\rm ref}), \varepsilon\}$. Step count and learning rate are selected via validation AUROC on held-out target samples. Baselines (LESS, TracIn, and ToV) share pool checkpoints, candidate budgets, and retraining schedules. LESS scores by cosine gradient similarity across checkpoints, TracIn by unnormalized gradient inner products, and ToV by candidate loss drop under one-step validation perturbations ($\alpha=0.1$).

\subsection{Details for Computer Vision Benchmarks}
\label{app:cv_details}

\paragraph{Task setup and training protocols.}
We evaluate binary cat-versus-dog classification on CIFAR-10~\cite{krizhevsky2009learning} with an ImageNet-pretrained~\cite{deng2009imagenet} \texttt{ResNet-18}~\cite{he2016deep} ($96\times96$ resolution). The target proxy contains 100 clean images. Candidate pools contain 10,000 images for full fine-tuning and 4,000 images for partial fine tuning on \texttt{layer4} and the classification head. The noisy regime injects 40\% symmetric label noise into candidate pools while keeping proxy and test sets clean.

TACS generates reference trajectories via Adam (LR $5\times10^{-5}$, batch size 32), scoring candidates via normalized loss reduction between passes 1 and 4. Baselines include pool-reference gradient alignment (LESS under full fine-tuning, TracIn under partial fine tuning), EL2N, GraNd~\cite{paul2023deeplearningdatadiet}, and pretrained feature retrieval. Retraining runs for 8 epochs (full fine-tuning) or 4 epochs (partial fine tuning) at LR $10^{-4}$. Performance sweeps across budgets are shown in Figures~\ref{fig:cv_budget_appendix} and~\ref{fig:cv_partialft_appendix}.

\begin{figure}[t]
\centering
\includegraphics[width=0.92\textwidth]{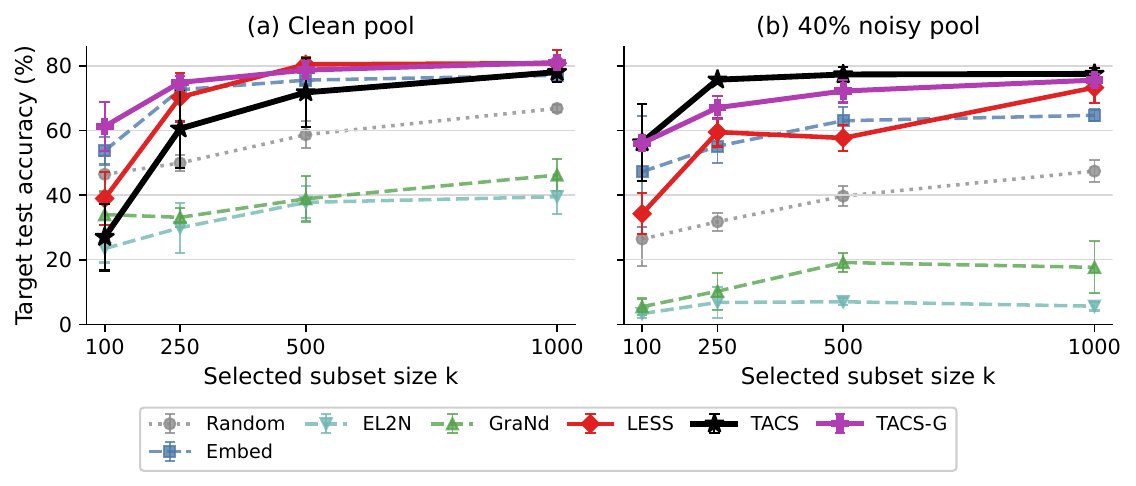}
\caption{\textbf{Binary CIFAR-10 full fine-tuning target accuracy across budgets} under clean and 40\% label-noise pools (mean $\pm$ std over 5 clean / 3 noisy seeds). TACS-G denotes class-grouped trajectories; LESS scores by cosine alignment.}
\label{fig:cv_budget_appendix}
\end{figure}

\begin{figure}[t]
\centering
\includegraphics[width=0.92\textwidth]{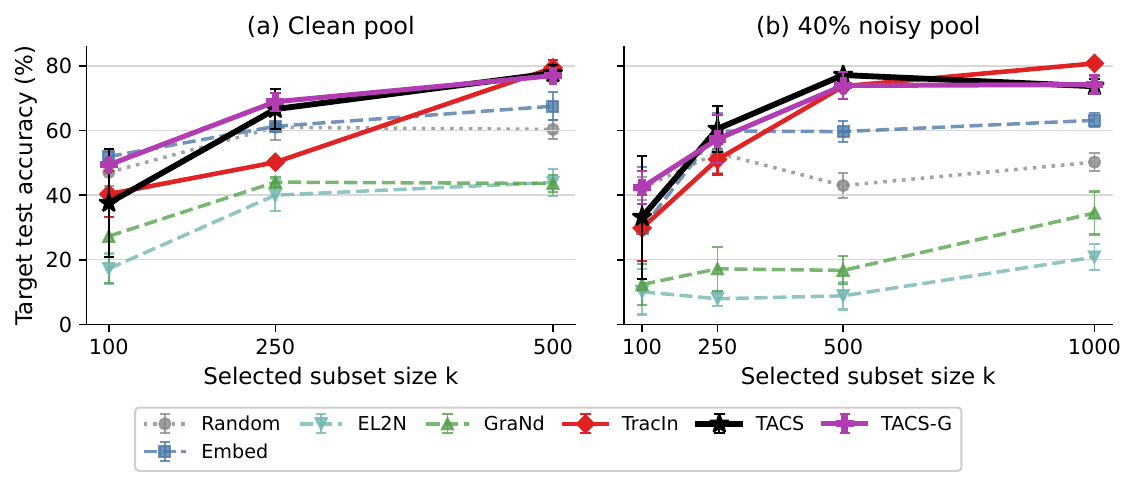}
\caption{\textbf{Binary CIFAR-10 partial fine tuning accuracy across budgets} under clean and 40\% label-noise pools (mean $\pm$ std over 2 clean / 3 noisy seeds).}
\label{fig:cv_partialft_appendix}
\end{figure}

\paragraph{Class-grouped target trajectories (TACS-G).}
\label{app:cv_grouping}
When target validation data covers multiple classes, joint warmups follow their mean gradient. The class-grouped variant, TACS-G, runs one trajectory per target class $g$ matching the step schedule of the joint path, scoring each candidate by its maximum normalized response: $s_{\mathrm G}(z) = \max_g s_g(z)$. Grouping sharpens target class boundaries in clean pools while exhibiting a slight trade-off in label-noise resilience (Figure~\ref{fig:cv_budget_appendix}).

\subsection{Details for Instruction Tuning}
\label{app:instruction_tuning_details}

\subsubsection{Dataset Statistics}
Candidate pools and target proxies are detailed in Table~\ref{tab:app_data_summary}. All data are standardized into the unified Tulu chat format~\cite{wang2023farcamelsgoexploring}. 

\begin{table}[htbp]
    \caption{\textbf{Dataset statistics} for candidate pools ($\mathcal{Z}_{\mathrm{pool}}$) and target proxies ($\mathcal{Z}_{\mathrm{val}}$).}
    \label{tab:app_data_summary}
    \small
    \centering
    \begin{tabular}{@{}lr c lrl@{}}
    \toprule
    \multicolumn{2}{c}{\textbf{Candidate Training Pools}} & \phantom{a} & \multicolumn{3}{c}{\textbf{Target Development Proxies}} \\
    \cmidrule{1-2} \cmidrule{4-6}
    \textbf{Source} & \textbf{Size} & & \textbf{Task} & \textbf{Size} & \textbf{Prompt Format} \\
    \midrule
    Flan V2~\cite{longpre2023flan} & 100,000 & & MMLU~\cite{hendrycks2020measure} & 285 & 5-shot \\
    COT~\cite{wei2023chainofthoughtpromptingelicitsreasoning} & 100,000 & & BBH~\cite{suzgun2022challenging} & 81 & 1-shot \\
    Oasst1~\cite{kopf2023openassistant} & 55,668 & & TyDiQA~\cite{clark2020tydi} & 9 & 1-shot \\
    Dolly~\cite{conover2023freedolly} & 15,011 & & & & \\
    \bottomrule
    \end{tabular}
\end{table}

\subsubsection{Model and Training Configurations}
All instruction tuning experiments use \texttt{Llama-3.2-3B}~\cite{dubey2024llama}, maximum sequence length 2,048, bfloat16 precision, and LoRA on attention projections with AdamW (zero weight decay). Checkpoint selection evaluates candidate checkpoints on validation data prior to test evaluation, reporting the optimal validation checkpoint alongside the final epoch-four checkpoint. Hyperparameters across stages are summarized in Table~\ref{tab:instruction_training_settings}.

\begin{table}[h]
\centering
\small
\caption{\textbf{Instruction-tuning training settings.} Source-specific learning rates for pool warmup and retraining are selected on random subsets from $\{1,2,5\}\times10^{-5}$ and shared across methods.}
\label{tab:instruction_training_settings}
\setlength{\tabcolsep}{3pt}
\begin{tabular}{@{}p{0.18\linewidth}p{0.24\linewidth}p{0.22\linewidth}p{0.30\linewidth}@{}}
\toprule
\textbf{Stage} & \textbf{Data Allocation} & \textbf{Adapter Configuration} & \textbf{Optimization Schedule} \\
\midrule
LESS / ToV warmup & Random 5\% pool subset & LoRA $r=128, \alpha=512$, drop 0.1 & 4 epochs, AdamW, source LR \\
TACS warmup & Full target proxy $\mathcal{Z}_{\mathrm{val}}$ & LoRA $r=1, \alpha=4$, drop 0.1 & AdamW, LR $10^{-5}$, $T=2/2/64$ for MMLU/BBH/TyDiQA \\
Subset retraining & Selected 5\% pool subset & LoRA $r=128, \alpha=512$, drop 0.1 & 4 epochs, effective batch 32, linear warmup \\
\bottomrule
\end{tabular}
\end{table}

\paragraph{Baseline specifications.}
LESS follows its official implementation ($d_{\rm proj}=8192$), projecting Adam-preconditioned gradients and aggregating cosine similarities across 4 pool warmup checkpoints. ToV applies validation perturbations ($\alpha=0.1$) around pool checkpoints, scoring candidates by maximum loss improvement across checkpoints.

\subsubsection{Paired Clean-Run Differences}
\label{app:clean_paired_differences}
Table~\ref{tab:clean_paired_differences} reports TACS minus baseline in display points. Intervals use paired evaluation-item resampling and condition on the three fitted seed runs.
\begin{table}[h]
\centering
\small
\caption{Paired clean-run differences. A positive value favors TACS; brackets are 95\% paired evaluation-item bootstrap intervals conditional on the three fitted seed runs.}
\label{tab:clean_paired_differences}
\setlength{\tabcolsep}{3.5pt}
\begin{tabular}{llcc}
\toprule
\textbf{Pool} & \textbf{Task} & \textbf{vs. LESS: mean [95\% CI]} & \textbf{vs. ToV: mean [95\% CI]} \\
\midrule
\multirow{3}{*}{\textbf{COT}} & MMLU & 0.42 [0.03, 0.80] & 0.63 [0.15, 1.11] \\
 & TyDiQA & -0.51 [-1.37, 0.34] & 6.56 [5.40, 7.66] \\
 & BBH & -2.01 [-3.73, -0.28] & -1.33 [-2.96, 0.34] \\
\midrule
\multirow{3}{*}{\textbf{Dolly}} & MMLU & -0.12 [-0.31, 0.05] & 0.12 [-0.08, 0.32] \\
 & TyDiQA & 2.12 [1.28, 2.94] & 3.07 [2.14, 3.96] \\
 & BBH & 0.06 [-1.14, 1.23] & 0.74 [-0.43, 1.94] \\
\midrule
\multirow{3}{*}{\textbf{Flan V2}} & MMLU & 1.09 [0.62, 1.55] & 3.83 [3.27, 4.37] \\
 & TyDiQA & -0.13 [-0.96, 0.72] & 0.73 [-0.24, 1.70] \\
 & BBH & 1.98 [0.22, 3.77] & 0.46 [-1.33, 2.28] \\
\midrule
\multirow{3}{*}{\textbf{Oasst1}} & MMLU & 0.21 [-0.09, 0.51] & 0.04 [-0.33, 0.41] \\
 & TyDiQA & 3.34 [2.44, 4.24] & 3.40 [2.57, 4.22] \\
 & BBH & 1.20 [-0.43, 2.78] & 0.09 [-1.57, 1.70] \\
\bottomrule
\end{tabular}
\end{table}

\subsubsection{Score-Band Monotonicity Diagnostic}
\label{app:score_band_diagnostic}

To evaluate whether TACS scores correlate monotonically with downstream utility across the candidate distribution, we sort the Dolly candidate pool by TACS score on TyDiQA and partition it into seven disjoint 5\% bands. Each band is retrained across three seeds (3, 7, 42). As shown in Table~\ref{tab:score_band_diagnostic}, downstream TyDiQA performance aligns closely with score percentiles across both checkpoint selection criteria ($\rho=0.964$, exact two-sided permutation $p=0.0028$), confirming that the scoring function acts as a reliable proxy for downstream transfer.

\begin{table}[t]
\centering
\small
\setlength{\tabcolsep}{6pt}
\caption{\textbf{Dolly$\rightarrow$TyDiQA score-band diagnostic on \texttt{Llama-3.2-3B}.} Disjoint 5\% candidate bands retrained across seeds 3, 7, and 42, evaluated on 1,713 TyDiQA questions.}
\label{tab:score_band_diagnostic}
\begin{tabular}{lrrr}
\toprule
\textbf{Score Percentile} & \textbf{Mean TACS Score} & \textbf{Fixed-Final F1} & \textbf{Validation-Selected F1} \\
\midrule
0--5\% (Highest) & $+0.03048$ & $63.60$ & $63.27$ \\
15--20\% & $+0.00799$ & $58.24$ & $57.34$ \\
30--35\% & $+0.00278$ & $55.47$ & $54.61$ \\
47.5--52.5\% & $+0.00031$ & $43.81$ & $44.03$ \\
65--70\% & $-0.00142$ & $37.54$ & $37.31$ \\
80--85\% & $-0.00386$ & $42.34$ & $40.53$ \\
95--100\% (Lowest) & $-0.01381$ & $27.33$ & $30.27$ \\
\bottomrule
\end{tabular}
\end{table}

\subsubsection{Ablation Settings}
\label{app:ablation_settings}
Ablation experiments use the clean Dolly pool with \texttt{Llama-3.2-3B} across three seeds. Figure~\ref{fig:main_ablations}a evaluates LoRA rank $r \in \{1, 4, 16, 64, 128\}$ ($\alpha=4r$) under the calibrated schedule ($\eta^*=10^{-5}$, $T^*=2$ for MMLU/BBH, $T^*=64$ for TyDiQA). Figure~\ref{fig:main_ablations}b evaluates trajectory horizons $T \in \{2, 4, 8, 12, 16, 32, 64\}$.

\subsubsection{Reference-Path Sensitivity Diagnostic}
\label{app:warmup_sensitivity_settings}

To empirically isolate reference-path dependence, we evaluate LESS selections on \texttt{Qwen2.5-1.5B} across structured warmups ($k$-means clusters, single-source warmups) versus random pool warmups on 750 reservoir examples (4 epochs, LoRA $r=128$). Table~\ref{tab:warmup_sensitivity_robustness} reports selection overlaps and total variation distances (TVD). In all cases, structured pool warmups induce substantially larger selection divergence than random draws, demonstrating that trajectory-based attribution is actively steered by the warmup distribution.

\begin{table}[htbp]
\centering
\small
\caption{\textbf{Reference-path sensitivity of LESS on \texttt{Qwen2.5-1.5B}.} Overlap measures the pairwise shared fraction among top-5\% selections (chance = $0.05$). TVD ratio compares composition distances between structured and random warmups.}
\label{tab:warmup_sensitivity_robustness}
\setlength{\tabcolsep}{3pt}
\renewcommand{\arraystretch}{1.12}
\begin{tabular*}{\linewidth}{@{\extracolsep{\fill}}lclcrrr@{}}
\toprule
 & & & & \multicolumn{2}{c}{\textbf{Top-5\% overlap}} & \\
\cmidrule(lr){5-6}
\textbf{Pool} & \textbf{Rank} & \textbf{Warmup family} & \textbf{Target} & Structured & Random & \textbf{TVD ratio} \\
\midrule
Dolly & 128 & $k$-means & TyDiQA & 0.236 & 0.654 & 5.5 \\
 & 128 & Diversity-matched & TyDiQA & 0.307 & 0.654 & 4.4 \\
 & 1 & $k$-means & TyDiQA & 0.610 & 0.898 & 6.0 \\
 & $1^{\dagger}$ & $k$-means & TyDiQA & 0.475 & 0.847 & 3.8 \\
\addlinespace[4pt]
Flan V2 & 128 & $k$-means & TyDiQA & 0.307 & 0.456 & -- \\
\addlinespace[4pt]
Four sources & 128 & Source-isolated & TyDiQA & 0.210 & 0.509 & 4.0 \\
 & 128 & Source-isolated & MMLU & 0.110 & 0.406 & 3.1 \\
 & 128 & Source-isolated & BBH & 0.097 & 0.417 & 3.7 \\
\bottomrule
\end{tabular*}
\par\smallskip
\begin{minipage}{\linewidth}
\footnotesize All warmups use LoRA; rank denotes $r$. $\dagger$ Norm-matched warmup.
\end{minipage}
\end{table}

\subsection{Ablations on Scoring Variants and Sequence Length}
\label{app:seq_len}

\paragraph{Length distribution across scoring criteria.}
Figure~\ref{fig:length_hist_appendix} illustrates length distributions of subsets selected by raw versus normalized endpoint drops across all instruction-tuning pools. Normalization effectively recalibrates candidate loss scales, confirming that the metric operates as an explicit loss-scale normalizer.

\begin{figure}[t]
\centering
\includegraphics[width=0.95\textwidth]{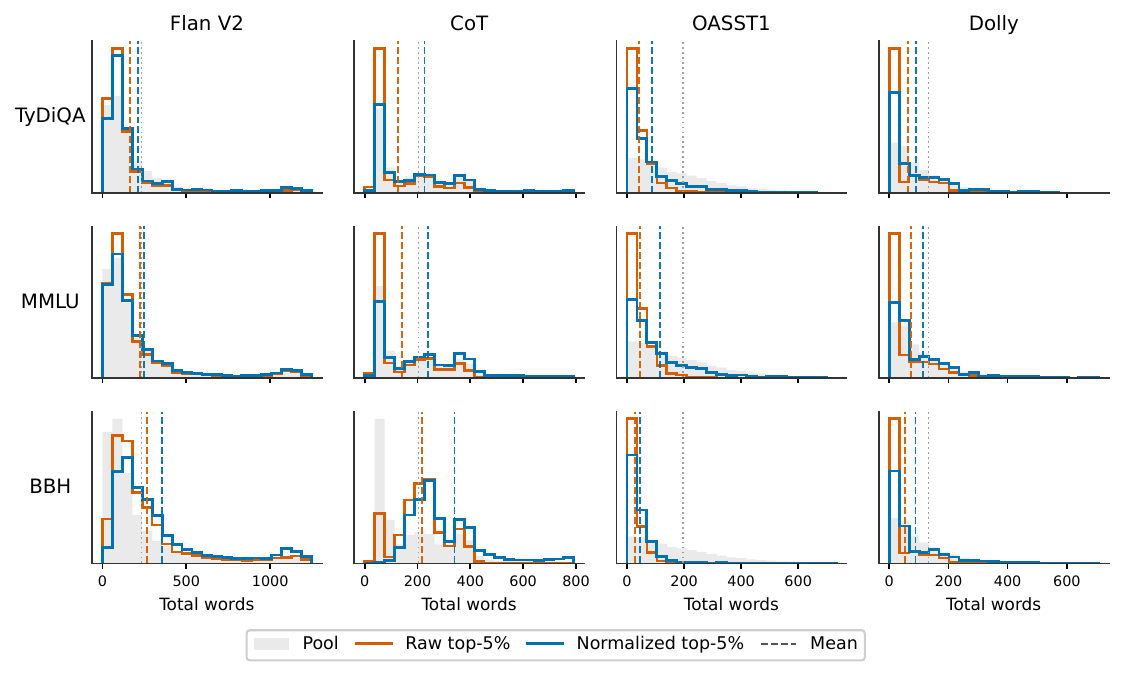}
\caption{\textbf{Selected-length distributions under raw and normalized endpoint loss differences.} Full candidate pools (gray) versus selected top-5\% subsets under raw and normalized TACS metrics.}
\label{fig:length_hist_appendix}
\end{figure}

\paragraph{Scoring formulation ablation.}
\label{app:score-variant-ablation}
Figure~\ref{fig:score-variant-ablation} evaluates normalization and baseline checkpoint choices ($a \in \{0, 1\}$) across synthetic and vision tasks:
\begin{equation}
    s^{\mathrm{raw}}_a(z) = \ell(\theta_a^{\val}; z) - \ell(\theta_T^{\val}; z), \qquad s^{\mathrm{norm}}_a(z) = \frac{\ell(\theta_a^{\val}; z) - \ell(\theta_T^{\val}; z)}{\max\{\ell(\theta_a^{\val}; z), \varepsilon\}}.
\end{equation}
In noisy CIFAR-10, scale normalization substantially improves clean target sample concentration, particularly under restricted budgets, showing that relative loss reduction actively penalizes high-loss corrupted instances.

\begin{figure}[t]
\centering
    \begin{minipage}{0.48\linewidth}
        \centering
        \includegraphics[width=\linewidth]{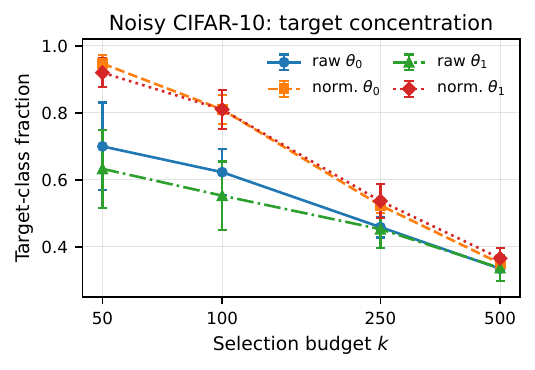}
        \vspace{-0.5em}
        \centerline{\small (a) CIFAR target-class fraction}
    \end{minipage}
    \hfill
    \begin{minipage}{0.48\linewidth}
        \centering
        \includegraphics[width=\linewidth]{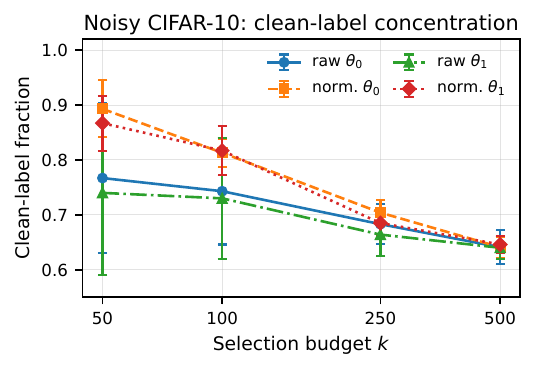}
        \vspace{-0.5em}
        \centerline{\small (b) CIFAR clean-label fraction}
    \end{minipage}

    \vspace{0.75em}

    \begin{minipage}{0.48\linewidth}
        \centering
        \includegraphics[width=\linewidth]{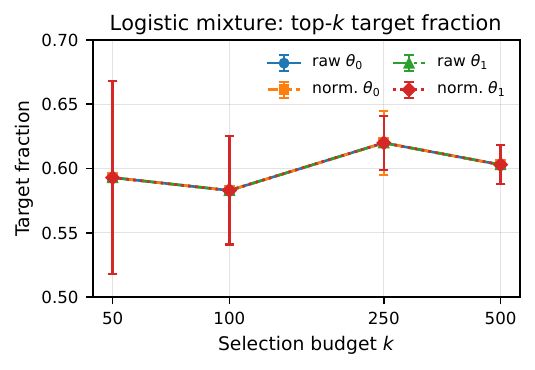}
        \vspace{-0.5em}
        \centerline{\small (c) Logistic target fraction}
    \end{minipage}
    \hfill
    \begin{minipage}{0.48\linewidth}
        \centering
        \includegraphics[width=\linewidth]{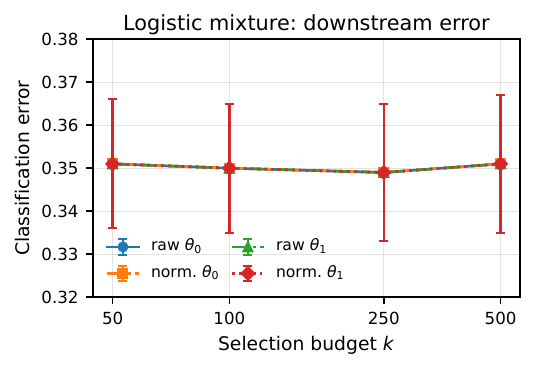}
        \vspace{-0.5em}
        \centerline{\small (d) Logistic downstream error}
    \end{minipage}
\caption{\textbf{Score-variant ablation in CIFAR-10 and logistic regression.} Evaluating clean and target-class sample purity under raw vs.\ normalized and $\theta_0$ vs.\ $\theta_1$ baseline checkpoints.}
\label{fig:score-variant-ablation}
\end{figure}

\subsection{Selection Cost Accounting Breakdown}
\label{app:selection_cost_accounting}

Table~\ref{tab:empirical_compute_cost} and Figure~\ref{fig:cost_pareto_clean} report selection costs in the seed-3 study, excluding downstream retraining and evaluation. Storage counts candidate scoring payloads transferred to disk.

The end-to-end selection compute decomposes as:
\begin{equation}
    C_{\mathrm{selection}} = C_{\mathrm{warmup/calib}} + C_{\mathrm{reference}} + C_{\mathrm{scoring}}.
\end{equation}
For LESS, selection totals 187.35 GPU-hours (185.53\,h for pool warmups and candidate gradient extraction, 1.83\,h for pairwise scoring). For ToV, selection totals 124.74 GPU-hours (4.24\,h for pool warmups, 120.50\,h for validation perturbations and candidate scoring). For TACS, selection totals 58.46 GPU-hours (7.54\,h for the 3-fold calibration grid, 0.37\,h for final target references, and 50.55\,h for candidate scoring).

\begin{figure}[t]
    \centering
    \includegraphics[width=0.92\textwidth]{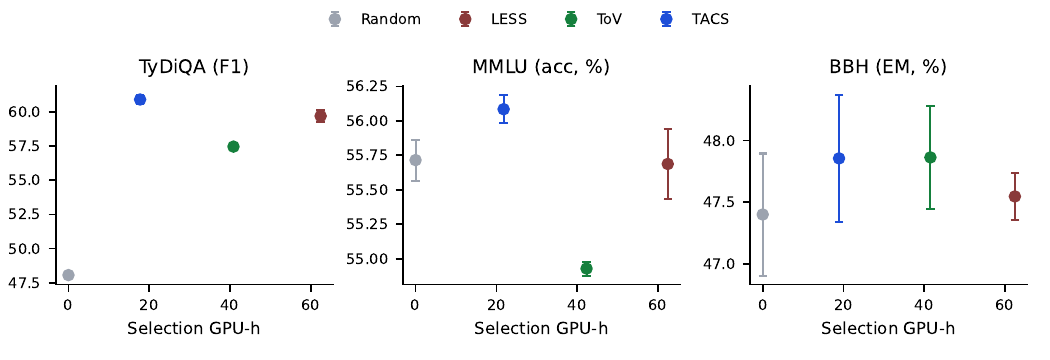}
    \caption{\textbf{Selection compute versus downstream utility} in instruction tuning ($m=4$ pools, $k=3$ targets). Points denote source-averaged three-seed means; error bars denote sample standard deviations. Common retraining and evaluation are excluded.}
    \label{fig:cost_pareto_clean}
\end{figure}

Figure~\ref{fig:cost_crossover_clean} evaluates the empirical cost models across varying numbers of candidate pools ($m$) and target tasks ($k$). TACS is the most efficient selector in multi-pool regimes, while LESS becomes advantageous when a static pool is amortized across an exceptionally large number of downstream targets.

\begin{figure}[ht]
    \centering
    \includegraphics[width=0.58\linewidth]{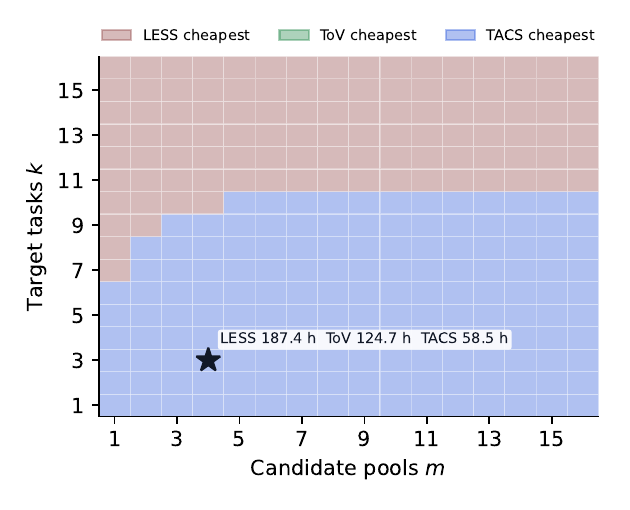}
    \caption{\textbf{Empirical cost crossover regimes.} Extrapolated compute boundaries as candidate pools ($m$) and downstream targets ($k$) scale under empirical cost models. The star marks the evaluated setup ($m=4, k=3$).}
    \label{fig:cost_crossover_clean}
\end{figure}

\subsection{Trajectory Calibration Protocol}
\label{sec:appendix_hyperparameters}

We calibrate learning rate $\eta$ and horizon $T$ via 3-fold cross-validation on proxy $\mathcal{Z}_{\mathrm{val}}$ using rank-one LoRA ($r=1, \alpha=4$). Across candidate grids $\eta \in \{1, 2, 5\} \times 10^{-5}$ and $T \in \{2, 4, 8, 12, 16, 32, 64\}$, we record endpoint gradients on the warmup split ($g_{W_j}$), held-out split ($g_{H_j}$), and a generic reference sample $P$ ($g_P$), maximizing mean specificity:
\begin{equation}
    C_j(\eta, T) = \cos(g_{W_j}, g_{H_j}) - \cos(g_{W_j}, g_P).
\end{equation}
Ties break in favor of shorter trajectories and smaller learning rates (Algorithm~\ref{alg:trajectory_calibration}). All tasks select $\eta^* = 10^{-5}$, with $T^* = 2$ for MMLU and BBH, and $T^* = 64$ for TyDiQA. Generic reference $P$ consists of 100 fixed Latin-script examples from Aya~\cite{singh2024aya}, verified to have zero string overlap with candidate pools.

\begin{algorithm}[ht]
\caption{$M$-Fold Trajectory Calibration}
\label{alg:trajectory_calibration}
\small
\begin{algorithmic}[1]
    \STATE \textbf{Input:} Target validation set $\mathcal{Z}_{\mathrm{val}}$, generic reference sample $P$, grids $H_\eta, H_T$
    \STATE Partition $\mathcal{Z}_{\mathrm{val}}$ into $M=3$ disjoint folds
    \FOR{each candidate learning rate $\eta \in H_\eta$}
        \FOR{each fold $j=1,\ldots,M$}
            \STATE Set held-out split $H_j \leftarrow \text{fold } j$, warmup split $W_j \leftarrow \mathcal{Z}_{\mathrm{val}} \setminus H_j$
            \STATE Train rank-1 LoRA on $W_j$ with step size $\eta$ up to $\max H_T$ epochs
            \FOR{each horizon $T \in H_T$}
                \STATE Compute endpoint gradients $g_{W_j}, g_{H_j}, g_P$ at checkpoint $\theta_{j,T}$
                \STATE $C_j(\eta,T) \leftarrow \cos(g_{W_j}, g_{H_j}) - \cos(g_{W_j}, g_P)$
            \ENDFOR
        \ENDFOR
        \STATE Compute mean specificity across folds: $\bar C(\eta,T) \leftarrow \frac{1}{M}\sum_{j=1}^M C_j(\eta,T)$
    \ENDFOR
    \STATE \textbf{return} $(\eta^*, T^*) = \arg\max_{\eta,T} \bar C(\eta,T)$ (breaking ties in favor of smaller $T$, then smaller $\eta$)
\end{algorithmic}
\end{algorithm}

\section{Proof of Reference-Induced Selection Uncertainty}
\label{app:reference_uncertainty}

We prove Theorem~\ref{thm:pool_warmup_tradeoff} for a fixed scored family $\mathcal Q=\{Q_u:u\in\mathcal U\}$. The warmup reference is random; the candidate family, target, checkpoint schedule, and scoring rule are fixed. Assume $\mathcal U$ is compact and $s(u;\xi)$ is continuous on $\mathcal U$ and $C^3$ in the reference and local feasible coordinates. On an open neighborhood of the compact convex reference set $\mathcal N$, suppose every score has a unique maximizer $v(\xi)$ in a $C^3$ feasible chart $u=\psi(a)$ with a negative-definite Hessian in $a$. Active constraints must have a stable active set. These conditions include regular smooth boundary maxima. Assume all empirical warmup means lie in $\mathcal N$ and have finite fourth moments.

\paragraph{Warmup covariance and selector response.}
For task mean $\xi$, averaging $R_u(w)=\frac12\|w-u\|_H^2$ gives gradient $H(w-\xi)$. Solving $\dot w=-H(w-\xi)$ from zero yields $w_t=B_t\xi$. Conditional on composition $C$, $\mathbb E[\widehat\mu_m\mid C]=\mu_C$ and $\operatorname{Cov}(\widehat\mu_m\mid C)=\Sigma_C/m$. Total covariance therefore gives
\begin{equation}\label{eq:reference:covariance}
\mathbb E\widehat\mu_m=\mu,\qquad \operatorname{Cov}(\widehat\mu_m)=\Omega_c+\Sigma_w/m.
\end{equation}
For uniform sampling without replacement from $M_C>1$ tasks, with $\Sigma_C=M_C^{-1}\sum_j(U_j-\mu_C)(U_j-\mu_C)^\top$, the conditional term becomes $(M_C-m)\Sigma_C/[m(M_C-1)]$ by expanding inclusion-indicator covariances.

In local coordinates let $\widetilde s(a;\xi)=s(\psi(a);\xi)$, $A_\xi=-\nabla^2_{aa}\widetilde s(a_\xi;\xi)\succ0$, and $C_\xi=\nabla^2_{a\xi}\widetilde s(a_\xi;\xi)$. The implicit function theorem applied to $\nabla_a\widetilde s(a_\xi;\xi)=0$ gives a $C^2$ stationary branch. Negative curvature preserves local maximality; compactness and uniqueness ensure that the branch is the global maximizer locally and that branches agree on overlaps. A finite cover of $\mathcal N$ bounds its first two derivatives. Differentiation of stationarity yields
\begin{equation}\label{eq:reference:response}
D_\xi a_\xi=A_\xi^{-1}C_\xi,\qquad J(\xi)=D\psi(a_\xi)A_\xi^{-1}C_\xi.
\end{equation}
For an interior optimum this is $J=-s_{uu}^{-1}s_{u\xi}$. For $s(u;\xi)=\mathcal A(u;B_{t_1}\xi,\ldots,B_{t_L}\xi,q)$, the chain rule gives $C_\xi=\sum_\ell\nabla^2_{aw_\ell}\widetilde{\mathcal A}(a_\xi;B_{t_1}\xi,\ldots,B_{t_L}\xi,q)B_{t_\ell}$, with $\widetilde{\mathcal A}(a;\cdot)=\mathcal A(\psi(a);\cdot)$. Thus warmup dynamics, the score's mixed response, and inverse selection curvature jointly determine $J$. A strictly increasing smooth score transformation $h(s(u;\xi),\xi)$ preserves $J$: at stationarity its negative Hessian and mixed derivative are both multiplied by $h_s>0$.

\paragraph{Finite-perturbation variance lemma.}
Let $F:\mathcal N\to\mathbb R^r$ be $C^2$, $\widehat\xi=\xi+\delta$, $\mathbb E\delta=0$, $\operatorname{Cov}(\delta)=\Omega$, and $\|D^2F(x)[a,b]\|\leq L_F\|a\|\|b\|$ along every relevant segment. With $L=DF(\xi)$, Taylor's formula gives $F(\xi+\delta)=F(\xi)+L\delta+r(\delta)$ and $\|r(\delta)\|\leq L_F\|\delta\|^2/2$. Expanding centered squared norms and bounding the cross and remainder terms gives
\begin{equation}\label{eq:reference:variance_bound}
\left|\operatorname{tr}\operatorname{Cov}(F(\widehat\xi))-\operatorname{tr}(L\Omega L^\top)\right|
\leq\|L\|_{\mathrm{op}}L_F\mathbb E\|\delta\|^3+\frac{L_F^2}{4}\mathbb E\|\delta\|^4.
\end{equation}
For independent $V,V'$, $\frac12\mathbb E\|V-V'\|_H^2=\operatorname{tr}(H\operatorname{Cov}(V))$. Apply Eq.~\eqref{eq:reference:variance_bound} to $F(\xi)=H^{1/2}v(\xi)$, $\delta=\widehat\mu_m-\mu$, and $\Omega=\Omega_c+\Sigma_w/m$. Apply it again to the scalar $f(\xi)=D(v(\xi))$, whose gradient is $g=J^\top H(v(\mu)-q)$. Since $\mathbb E\|\delta\|^3\leq\epsilon_m^3$ and $\mathbb E\|\delta\|^4=\epsilon_m^4$, these two applications prove Eq.~\eqref{eq:reference_uncertainty_law}. The leading dispersion is positive exactly when $H^{1/2}J\Omega^{1/2}\ne0$; the leading risk variance is positive exactly when $\Omega^{1/2}g\ne0$. At $v(\mu)=q$, risk variance can be fourth order while selection dispersion is second order.

\paragraph{Exact limits and mean identity.}
Construct each warmup from an infinite conditional iid sequence. The conditional strong law gives $\widehat\mu_m\to\mu_C$ almost surely. Continuity of $v$, compactness of $\mathcal U$, and dominated convergence for first and second moments imply
\begin{equation}\label{eq:reference:limit}
I_m\to\operatorname{tr}(H\operatorname{Cov}_C(v(\mu_C))),\qquad
\operatorname{Var}D(\widehat v_m)\to\operatorname{Var}_C D(v(\mu_C)).
\end{equation}
These limits are positive precisely when the selected optimum and its target risk, respectively, are nonconstant almost surely across compositions. For any square-integrable selected optimum $V$, expanding $V-q=(\mathbb EV-q)+(V-\mathbb EV)$ gives $\mathbb ED(V)=D(\mathbb EV)+\frac12\operatorname{tr}(H\operatorname{Cov}(V))$. Taking $V=\widehat v_m$ proves Eq.~\eqref{eq:mean_dispersion_identity}. The deterministic shift also satisfies $v(\mu)-v(q)=\int_0^1J(q+t(\mu-q))(\mu-q)\,dt$ on a regular connecting segment.

\paragraph{Finite target proxy.}
For an iid $n$-sample target proxy $\widehat q_n$ with covariance $\Sigma_D/n$, finite fourth moments give $\mathbb E\|\widehat q_n-q\|^4=O(n^{-2})$ by coordinatewise fourth-moment expansion and third absolute moments $O(n^{-3/2})$ by H\"older's inequality. Applying Eq.~\eqref{eq:reference:variance_bound} at $q$ gives
\begin{equation}\label{eq:reference:target_proxy}
\begin{aligned}
\operatorname{tr}(H\operatorname{Cov}(v(\widehat q_n)))&=n^{-1}\operatorname{tr}(HJ_q\Sigma_DJ_q^\top)+O(n^{-3/2}),\\
\operatorname{Var}D(v(\widehat q_n))&=n^{-1}g_q^\top\Sigma_Dg_q+O(n^{-3/2}).
\end{aligned}
\end{equation}
where $J_q=Dv(q)$ and $g_q=J_q^\top H(v(q)-q)$. For local regularity, restrict Taylor expansion to a regular ball: fourth moments make excursion probability $O(n^{-2})$, while compactness bounds the omitted outcomes. A fixed target proxy has zero reference variation conditional on that proxy. If the same proxy enters both path and scoring target, use $v(\xi,r)$ and replace $J_q$ by $\partial_\xi v(q,q)+\partial_rv(q,q)$. This completes the proof of Theorem~\ref{thm:pool_warmup_tradeoff}.

\subsection{Distributional disagreement and score scope}
If $Q_u$ has density $p_u$ and $\xi\mapsto p_{v(\xi)}$ is twice differentiable as an $L^1$-valued map, write $h(x)=\nabla_\xi p_{v(\xi)}(x)$. For independent centered Gaussian $Z,Z'$ with covariance $\Gamma$, an $L^1$ Taylor expansion and $\operatorname{TV}(P,Q)=\frac12\|p-q\|_1$ give
\begin{equation}\label{eq:reference:tv}
\mathbb E\operatorname{TV}(Q_{v(\xi+\varepsilon Z)},Q_{v(\xi+\varepsilon Z')})
=\frac{\varepsilon}{\sqrt\pi}\int\sqrt{h(x)^\top\Gamma h(x)}\,d\lambda(x)+O(\varepsilon^2).
\end{equation}
The remainder follows from the bounded second derivative and reverse triangle inequality; Gaussian excursions outside the regular neighborhood are negligible with a probability-valued extension. For iid mean perturbations, $\varepsilon=m^{-1/2}$ gives order $m^{-1/2}$ disagreement when the coefficient is nonzero. A sample-mean CLT gives the same coefficient under uniform integrability. For category proportions $\pi(u)$, letting $b_k^\top$ be the rows of $D(\pi\circ v)(\xi)$, the same argument gives $\mathbb E\operatorname{TV}_{\rm cat}=\varepsilon\sum_k\sqrt{b_k^\top\Gamma b_k}/\sqrt\pi+O(\varepsilon^2)$. If square-root densities are $C^2$ as $L^2$ maps on common support, squared Hellinger distance has leading term $\varepsilon^2\operatorname{tr}(F_Q(v(\xi))J\Gamma J^\top)/4$, where $F_Q$ is Fisher information; the remainder is $O(\varepsilon^3+\varepsilon^4)$.

The score class includes normalized checkpoint alignments $\sum_\ell\alpha_\ell\langle H(w_\ell-u),H(w_\ell-q)\rangle/[\|H(w_\ell-u)\|^\rho\|H(w_\ell-q)\|^\nu]$, $w_\ell=B_{t_\ell}\xi$, on domains with nonzero denominators (including $0\leq\rho\leq2$, $\nu\geq0$). It also includes smooth denominator regularizations, endpoint loss changes, nonlinear checkpoint aggregations, and preconditioned bilinear scores. Each still requires a regular optimum. A maximum across checkpoint scores is smooth on a stable unique active branch; ties require directional analysis. A reference-independent score has $J=0$.

For a constrained example, take $\mathcal U=\{u:\|u\|=R\}$ and $B=\sum_\ell\alpha_\ell B_{t_\ell}$ with $\sum_\ell\alpha_\ell=1$. Unnormalized checkpoint alignment reduces, up to a candidate-common term, to $s(u;\xi)=u^\top b(\xi)$ with $b(\xi)=H^2(q-B\xi)$. For $b\ne0$, $v(\xi)=Rb/\|b\|$ and $J(\xi)=-R(I-ee^\top)H^2B/\|b\|$, where $e=b/\|b\|$. The tangent Hessian is $-\|b\|I/R$, proving regularity. Sensitivity increases as the residual ranking direction weakens, subject to remaining away from $b=0$.

\section{Distribution Selection in the Linearized Model}
\label{app:distribution_selection}
\label{sec:selected_data}

We analyze the population preferences behind
Theorems~\ref{thm:selection_risk} and~\ref{thm:target_modes} in the linearized
model of Section~\ref{sec:theory:linear}. The shared covariance controls how
quickly target directions are learned, while candidate and target signals are
represented by \(u\) and \(q\). This setting separates direction, scale, and
orthogonal components without the combinatorics of finite-subset selection;
optimizer corrections and checkpoint aggregation remain outside the theorem.
It is also closely related to linearized analyses of language-model
fine-tuning~\citep{afzal2026linearizationexplainsfinetuninglarge}.

\subsection{Detailed Analysis and Proofs}
\label{app:distribution_selection:details}

\paragraph{Which candidate distributions are preferred?}
Let target gradient flow reach $v_\tau$ after warmup time $\tau>0$.
Write a candidate optimum as $u=av_\tau+b$, where
$b^\top Hv_\tau=0$. 
With an inactive denominator floor, normalized endpoint scoring gives, for $u\neq 0$,
\begin{equation}
\label{eq:normalized_geometry}
s(u;v_\tau)
=
1-\frac{\|u-v_\tau\|_H^2}{\|u\|_H^2}
=
\frac{(2a-1)\|v_\tau\|_H^2}
{a^2\|v_\tau\|_H^2+\|b\|_H^2}.
\end{equation}
For any $0<c<1$, completing the square gives
\[
s(u;v_\tau)\ge c
\quad\Longleftrightarrow\quad
\left\|u-\frac{v_\tau}{c}\right\|_H^2
\le
\frac{1-c}{c^2}\|v_\tau\|_H^2.
\]
Thus, every positive-score superlevel set is bounded and shrinks to
$\{v_\tau\}$ as $c\uparrow 1$.

The unique maximizer over this task family is $u=v_\tau$. At fixed
$a>1/2$, increasing $\|b\|_H$ lowers the score. TACS therefore favors
candidate distributions whose optima approach the target warmup endpoint
and penalizes orthogonal components in the geometry of the loss. Both the
preferred direction and scale are determined by the target trajectory.

\paragraph{Dot-product alignment and raw loss drop.}
For the fixed target endpoint $v_\tau$, dot-product alignment is
$d(u,v_\tau)=u^\top Hv_\tau=a\|v_\tau\|_H^2$.
Expanding the endpoint loss drop gives
\[
    \begin{aligned}
    \Delta(u;v_\tau)
    &= R_u(0)-R_u(v_\tau)
     = \tfrac12\|u\|_H^2-\tfrac12\|u-v_\tau\|_H^2 \\
    &= d(u,v_\tau)-\tfrac12\|v_\tau\|_H^2
     = (a-\tfrac12)\|v_\tau\|_H^2.
    \end{aligned}
\]
The subtracted term is constant across candidates for fixed $v_\tau$,
so both scores induce identical rankings, including ties. Both are independent of $b$ at fixed $a$ and tend to infinity as
$a\to\infty$. These preferences hold over the unrestricted candidate
family; a finite pool or a constraint on candidate signal magnitude
can limit the attainable scores. For TACS, the same increase in
orthogonal signal raises the normalization denominator while leaving
the raw loss drop unchanged, explaining the penalty when $a>1/2$.

Cosine alignment to the same endpoint is
\[
    c(u,v_\tau)=\frac{u^\top Hv_\tau}{\|u\|_H\|v_\tau\|_H}.
\]
By Cauchy--Schwarz it is at most one, with equality precisely when
$u=\alpha v_\tau$, $\alpha>0$. It therefore fixes the direction but leaves
the candidate scale undetermined.
\paragraph{Connection to trajectory-integrated gradient alignment.}
Along target gradient flow $\dot w_t=-\nabla R_q(w_t)$, $w_0=0$,
the chain rule gives the exact identity
\[
    \int_0^\tau
    \langle\nabla R_u(w_t),\nabla R_q(w_t)\rangle\,dt
    =R_u(0)-R_u(v_\tau)
    =d(u,v_\tau)-\tfrac12\|v_\tau\|_H^2.
\]
Thus integrated gradient dot products along this target path and endpoint
dot-product alignment give the same ranking in the quadratic model.
The instantaneous gradient inner product uses the ordinary parameter-space
metric; $d$ is the resulting endpoint expression in the $H$ metric.
This identity does not equate an arbitrary weighted checkpoint sum on a
pool-trained path with an endpoint score. Cosine-normalized gradients also
do not satisfy this loss-difference identity.

Actual single-checkpoint gradient cosine is a related, distinct comparison.
At a reference checkpoint $w=\gamma q+\xi\neq q$, where
$\xi^\top Hq=0$, the candidate and target gradients are $H(w-u)$
and $H(w-q)$. Since $H$ is invertible, their cosine reaches one
precisely when
\begin{equation}
\label{eq:discussion:cosine_preference}
    u=w+\alpha(q-w)
    =
    [\gamma+\alpha(1-\gamma)]q+(1-\alpha)\xi,
    \qquad \alpha>0.
\end{equation}
Thus, equally scoring candidates can carry different multiples of the
reference residual $\xi$. This family includes the target optimum at
$\alpha=1$; the cosine criterion alone leaves the scale undetermined.
The calculation describes a single-checkpoint gradient cosine score, before
the optimizer adjustments and aggregation used by LESS. Its preferred ray
is generally different from the endpoint-alignment ray in
Theorem~\ref{thm:selection_risk}: $v_\tau$ need not be parallel to $q$.

These preferences matter through the update induced by the selected tasks.
For a uniformly weighted collection of candidate distributions $S$,
write its mean optimum as $\bar u_S=\bar a_Sq+\bar b_S$, with
$\bar b_S^\top Hq=0$. Training on their average risk converges to
$\bar u_S$, whose target risk is
\begin{equation}
\label{eq:discussion:linear_target_risk}
    R_q(\bar u_S)
    =
    \frac12(\bar a_S-1)^2\|q\|_H^2
    +
    \frac12\|\bar b_S\|_H^2.
\end{equation}
Target performance therefore depends on both the aligned scale and
the mean orthogonal residual. TACS penalizes residuals relative to
$v_\tau$, whereas downstream risk is measured relative to $q$.
Understanding which data it favors thus also requires understanding
how well the warmup endpoint captures the target.

\paragraph{Which target components guide selection?}
Let $He_i=\lambda_i e_i$, with orthonormal $e_i$ and
$\lambda_1\geq\cdots\geq\lambda_d>0$, and write $q_i=e_i^\top q$.
Starting from zero, target gradient flow reaches
\begin{equation}
\label{eq:discussion:spectral_endpoint}
    v_\tau=(I-e^{-H\tau})q
    =\sum_{i=1}^d(1-e^{-\lambda_i\tau})q_i e_i
\end{equation}
at time $\tau$.

Thus, the distribution maximizing TACS matches a spectrally filtered
target. Components with $\lambda_i\tau\gg1$ are largely preserved,
while those with $\lambda_i\tau\ll1$ are attenuated. TACS favors
candidate distributions matching the target components learned during
warmup. A component can be relevant to the target yet receive little
support from a short reference trajectory if it lies in a slowly learned
direction.

Training to convergence on the score-maximizing distribution
$P_{v_\tau}$ yields $w=v_\tau$, with relative target risk
\begin{equation}
\label{eq:discussion:spectral_target_risk}
    \frac{R_q(v_\tau)}{R_q(0)}
    =\frac{\sum_{i=1}^d\lambda_i q_i^2e^{-2\lambda_i\tau}}
    {\sum_{i=1}^d\lambda_i q_i^2}.
\end{equation}
To characterize when this preferred distribution captures the target
well, define the cumulative target power in the leading eigendirections,
\begin{equation}
\label{eq:discussion:target_power}
    C_q(k)=\frac{\sum_{i=1}^k\lambda_iq_i^2}
    {\sum_{i=1}^d\lambda_iq_i^2}.
\end{equation}
For a fixed $H$, if one target has at least as large a $C_q(k)$ at every $k$ as
another, its relative risk in
Equation~\eqref{eq:discussion:spectral_target_risk} is no larger at any
warmup duration. This follows by summation by parts, since
$e^{-2\lambda_i\tau}$ increases with $i$.
The condition expresses task--model alignment through cumulative
spectral power~\citep{canatar2021spectral}: target signal concentrated
in rapidly learned directions makes the score-preferred distribution
a closer approximation to the target. Concentration in a few directions
alone is insufficient if those directions are learned slowly.

The score's penalty on orthogonal components is therefore relative to
the warmup endpoint, which inherits the target trajectory's spectral
preference. All target components in this model are representable and
eventually learned; the attenuation arises from finite warmup.
The score-maximizing distribution need not minimize downstream target
risk, which also depends on the aggregate update from selected tasks.
These results characterize population-level distribution preferences;
individual-example scoring with cross-entropy introduces additional
loss-scale and curvature effects.

\paragraph{Proof of Theorem~\ref{thm:selection_risk}.}
The population gradient is $\nabla R_u(w)=H(w-u)$.
Consequently, target gradient flow solves
$\dot v=-H(v-q)$, $v(0)=0$, yielding
Equation~\eqref{eq:discussion:spectral_endpoint}.
Substituting the population risks into the normalized endpoint score
and using $b^\top Hv_\tau=0$ gives
Equation~\eqref{eq:normalized_geometry}.
The score is at most one, with equality only at $u=v_\tau$;
its numerator in the orthogonal decomposition is positive for $a>1/2$,
so increasing $\|b\|_H^2$ strictly decreases it.

For cosine alignment, Cauchy--Schwarz in the $H$ inner product gives
$c(u,v)\leq1$, with equality exactly when $u=\alpha v$, $\alpha>0$.
For dot-product alignment, $b^\top Hv=0$ gives
$d(u,v)=a\|v\|_H^2$, independent of the orthogonal residual and unbounded
as $a\to\infty$. Expanding the quadratic risks gives
$R_u(0)-R_u(v)=d(u,v)-\frac12\|v\|_H^2$.
For fixed $v$, the subtracted term is independent of $u$, proving
the claimed ranking equivalence with raw loss drop.
Writing $q=cv+r$ and $u=av+b$, with $r,b$ both $H$-orthogonal
to $v$, gives
$D(u)=\frac12(a-c)^2\|v\|_H^2+\frac12\|b-r\|_H^2$.
For a mixture of candidate distributions, the average gradient is
$H(w-\bar u_S)$, so gradient flow converges to $\bar u_S$.
Expanding $\|\bar u_S-q\|_H^2$ in its target-aligned and orthogonal
components gives Equation~\eqref{eq:discussion:linear_target_risk}.

\paragraph{Proof of Theorem~\ref{thm:target_modes}.}
The gradient-flow solution above gives the preferred optimum.
Substitution into $R_q$ gives the stated relative risk.
Let $p_i=\lambda_iq_i^2/\|q\|_H^2$ and
$h_i=e^{-2\lambda_i\tau}$.
The relative risk in Equation~\eqref{eq:discussion:spectral_target_risk}
is $\sum_i p_i h_i$. Since $C_q(k)=\sum_{i\leq k}p_i$ and
$C_q(d)=1$, summation by parts gives
\[
    \sum_{i=1}^d p_i h_i
    =h_d-\sum_{k=1}^{d-1}C_q(k)(h_{k+1}-h_k).
\]
Each difference $h_{k+1}-h_k$ is nonnegative. Thus, for fixed $H$,
larger cumulative target power at every $k$ implies no larger relative
risk at every warmup duration.

\end{document}